\documentclass[11pt]{article}

\usepackage{subfig}
\usepackage{tikz}
\usepackage{dsfont}
\usepackage{pgf,pgfplots}
\usepackage{mathrsfs}

\def\colorful{0}

\usepackage{amsthm,amsfonts,amsmath,amssymb,epsfig,color,float,graphicx,verbatim, enumitem}
\usepackage{multirow}
\usepackage{algorithm}
\usepackage[noend]{algpseudocode}

\newif\ifhyper\IfFileExists{hyperref.sty}{\hypertrue}{\hyperfalse}
\hypertrue
\ifhyper\usepackage{hyperref}\fi

\usepackage{enumitem}

\usepackage{environ}

\newcommand{\repeattheorem}[1]{\begingroup
  \renewcommand{\thetheorem}{\ref{#1}}\expandafter\expandafter\expandafter\theorem
  \csname reptheorem@#1\endcsname
  \endtheorem
  \endgroup
}

\newcommand{\repeatlemma}[1]{\begingroup
  \renewcommand{\thelemma}{\ref{#1}}\expandafter\expandafter\expandafter\lemma
  \csname replemma@#1\endcsname
  \endlemma
  \endgroup
}

\NewEnviron{replemma}[1]{\global\expandafter\xdef\csname replemma@#1\endcsname{\unexpanded\expandafter{\BODY}}\expandafter\lemma\BODY\unskip\label{#1}\endlemma
}

\NewEnviron{reptheorem}[1]{\global\expandafter\xdef\csname reptheorem@#1\endcsname{\unexpanded\expandafter{\BODY}}\expandafter\theorem\BODY\unskip\label{#1}\endtheorem
}

\def\nnewcolor{0}

\ifnum\nnewcolor=1
\newcommand{\nnew}[1]{{\color{red} #1}}
\fi
\ifnum\nnewcolor=0
\newcommand{\nnew}[1]{#1}
\fi

\ifnum\colorful=1

\newcommand{\newblue}[1]{{\color{blue} #1}}
\else

\newcommand{\newblue}[1]{{#1}}
\fi

\newtheorem{theorem}{Theorem}[section]

\newtheorem{lemma}[theorem]{Lemma}
\newtheorem{informal theorem}[theorem]{Theorem (informal statement)}

\newtheorem{proposition}[theorem]{Proposition}

\newtheorem{claim}[theorem]{Claim}
\newtheorem{fact}[theorem]{Fact}

\newtheorem{remark}[theorem]{Remark}

\theoremstyle{definition}
\newtheorem{definition}[theorem]{Definition}
\newcommand{\eqdef}{\stackrel{{\mathrm {\footnotesize def}}}{=}}

\newtheorem{innercustomgeneric}{\customgenericname}
\providecommand{\customgenericname}{}
\newcommand{\newcustomtheorem}[2]{\newenvironment{#1}[1]
  {\renewcommand\customgenericname{#2}\renewcommand\theinnercustomgeneric{##1}\innercustomgeneric
  }
  {\endinnercustomgeneric}
}

\newcustomtheorem{customlem}{Lemma}
\newcustomtheorem{customclm}{Claim}
\newcustomtheorem{customfc}{Fact}

\newcommand{\lp}{\left}
\newcommand{\rp}{\right}

\newcommand\snorm[2]{\left\| #2 \right\|_{#1}}
\renewcommand\vec[1]{\mathbf{#1}}
\DeclareMathOperator*{\pr}{\mathbf{Pr}}
\DeclareMathOperator*{\E}{\mathbf{E}}

\def\d{\mathrm{d}}
\newcommand{\normal}{\mathcal{N}}

\DeclareMathOperator*{\argmin}{argmin}

\newcommand{\tr}{\mathrm{tr}}
\newcommand{\bx}{\mathbf{x}}
\newcommand{\by}{\mathbf{y}}

\newcommand{\bw}{\mathbf{w}}

\newcommand{\err}{\mathrm{err}}

\newcommand{\R}{\mathbb{R}}

\newcommand{\Z}{\mathbb{Z}}
\newcommand{\N}{\mathbb{N}}
\newcommand{\eps}{\epsilon}

\newcommand{\poly}{\mathrm{poly}}
\newcommand{\var}{\mathbf{Var}}

\newcommand{\sgn}{\mathrm{sign}}
\newcommand{\sign}{\mathrm{sign}}

\newcommand{\opt}{\mathrm{OPT}}
\newcommand{\D}{\mathcal{D}}

\newcommand{\Ind}{\mathds{1}}
\newcommand{\1}{\Ind}

\newcommand{\littlesum}{\mathop{\textstyle \sum}}

\newcommand{\citet}{\cite}
\newcommand{\citep}{\cite}
\usetikzlibrary{angles, quotes}
\usepackage{pgfplots}

\newcommand{\x}{\vec x}
\DeclareMathOperator\erf{erf}
\newcommand{\genf}{\rho}
\title{Agnostic Proper Learning of Halfspaces under Gaussian Marginals}

\author{
Ilias Diakonikolas\thanks{Supported by NSF Award CCF-1652862 (CAREER), a Sloan Research Fellowship, and
a DARPA Learning with Less Labels (LwLL) grant.}\\
UW Madison\\
{\tt ilias@cs.wisc.edu}\\
\and
Daniel M. Kane\thanks{Supported by NSF Award CCF-1553288 (CAREER) and a Sloan Research Fellowship.}\\
UC San-Diego \\
{\tt dakane@ucsd.edu}\\
\and
Vasilis Kontonis\\
UW Madison\\
{\tt kontonis@wisc.edu }\\
\and
Christos Tzamos\\
UW Madison\\
{\tt tzamos@wisc.edu}
\and
Nikos Zarifis\thanks{Supported in part by NSF Award CCF-1652862 (CAREER) and a DARPA Learning with Less Labels (LwLL) grant.}\\
UW Madison\\
{\tt zarifis@wisc.edu}\\
}

\begin{document}

\maketitle

\begin{abstract}
We study the problem of agnostically learning halfspaces under the Gaussian distribution. 
Our main result is the {\em first proper} learning algorithm for this problem 
whose sample complexity and computational complexity qualitatively match those 
of the best known improper agnostic learner. 
Building on this result, we also obtain the first proper polynomial-time approximation scheme (PTAS) 
for agnostically learning homogeneous halfspaces. Our techniques naturally extend to 
agnostically learning linear models with respect to other non-linear activations, 
yielding in particular the first proper agnostic algorithm for ReLU regression.
\end{abstract}

\setcounter{page}{0}
\thispagestyle{empty}
\newpage

\section{Introduction}\label{sec:intro}

\subsection{Background and Motivation} \label{ssec:background}

Halfspaces, or Linear Threshold Functions (LTFs), are Boolean functions $f: \R^d \to \{ \pm 1\}$ of
the form $f(\bx) = \sgn(\langle \bw, \bx \rangle - t)$, for some $\bw \in \R^d$ (known as the weight vector)
and $t \in \R$ (known as the threshold). The function $\sign: \R \to \{ \pm 1\}$ is defined
as $\sgn(u)=1$ for $u \geq 0$ and $\sgn(u)=-1$ otherwise.
Halfspaces have arguably been {\em the} most extensively studied
concept class in machine learning over the past six
decades~\citep{MinskyPapert:68, CristianiniShaweTaylor:00}.
The problem of learning halfspaces (in various models) is as old as the field of machine learning,
starting with the Perceptron algorithm~\citep{Rosenblatt:58, Novikoff:62},
and has been one of the most influential problems in the field  with techniques
such as SVMs~\citep{Vapnik:98} and AdaBoost~\citep{FreundSchapire:97} coming out of this study.

Here we study the task of learning halfspaces in the
{\em agnostic framework}~\citep{Haussler:92, KSS:94}, which models
the phenomenon of learning from adversarially labeled data. While
halfspaces are efficiently learnable
in the presence of consistently labeled examples (see, e.g.,~\cite{MT:94})
--- i.e., in Valiant's original PAC model~\citep{val84} ---
even {\em weak} agnostic learning is computationally hard
without distributional assumptions~\citep{GR:06, FGK+:06short, Daniely16}.
To circumvent this computational intractability, a line of work has focused on
the {\em distribution-specific} agnostic PAC model
--- where the learner has a priori information about the distribution on examples.
In this setting, computationally efficient noise-tolerant learning algorithms are
known~\citep{KKMS:08, KLS09, ABL17, Daniely15, DKS18a, DKTZ20c}
with various time-accuracy tradeoffs.

\begin{definition}[Distribution-Specific Agnostic Learning] \label{def:agnostic-ds}
Let $\mathcal{C}$ be a class of Boolean-valued functions on $\R^d$.
Given i.i.d.\ labeled examples $(\bx, y)$ from a distribution $\D$ on
$\R^d \times \{\pm 1\}$, such that the marginal distribution
$\D_{\bx}$ is promised to lie in a known distribution family $\mathcal{F}$
and no assumptions are made on the labels,
the goal of the learner is to output a hypothesis $h: \R^d \to \{\pm 1\}$
with small misclassification error, $\err_{0-1}^{\D}(h) \eqdef \pr_{(\bx, y) \sim \D}[h(\bx) \neq y]$,
as compared to the optimal misclassification error, $\opt \eqdef \inf_{g \in \mathcal{C}} \err_{0-1}^{\D}(g)$,
by any function in the class.
\end{definition}

\nnew{Throughout this paper, we will focus on the natural and well-studied case
that the underlying distribution on examples is the standard multivariate Gaussian distribution
$\normal( \boldsymbol 0, \vec I)$.}

Some additional comments are in order on Definition~\ref{def:agnostic-ds}.
In {\em improper} learning, the only assumption about the hypothesis $h$
is that it is {\em polynomially evaluable}. In other words, we assume that
$h \in \mathcal{H}$,  where $\mathcal{H}$ is a (potentially complex) class
of polynomially evaluable functions. In contrast, in {\em proper} learning
we have the additional requirement that the hypothesis $h$ is proper, i.e., $h \in \mathcal{C}$.
These notions of learning are essentially equivalent in terms of sample complexity,
but not always equivalent in terms of computational complexity. In particular, there exist
concept classes that are efficiently improperly learnable, while proper learning is computationally hard.

\newblue{
The classical $L_1$-polynomial regression algorithm of~\cite{KKMS:08} agnostically learns
halfspaces under the Gaussian distribution, within error $\opt+\eps$, with 
sample complexity and runtime of $d^{\poly(1/\eps)}$. 
On the lower bound side, recent work has provided evidence
that this complexity cannot be improved. Specifically,~\cite{DKZ20, GGK20, DKPZ21} 
obtained Statistical Query (SQ) lower bounds of $d^{\poly(1/\eps)}$ for this problem.
That is, the complexity of this learning problem is well-understood.

The polynomial regression algorithm~\cite{KKMS:08} is the only known agnostic learner
for halfspaces and is inherently {\em improper}: instead of a halfspace, 
its output hypothesis is a degree-$k$ polynomial threshold function (PTF), i.e., the sign of a degree-$k$ polynomial, where $k = \poly(1/\eps)$.  For the corresponding proper learning problem, 
prior to the present work, no non-trivial computational upper bound was known.

\paragraph{Importance of Proper Learning.} While an improper hypothesis suffices for the purpose
of prediction, an improper learner comes with some disadvantages. 
In our context, having such a complex output hypothesis requires spending $d^{\poly(1/\eps)}$ time 
for even evaluating the hypothesis on a single example. Moreover, storing the hypothesis 
function requires keeping track of the $d^{\poly(1/\eps)}$ coefficients defining 
the corresponding polynomial. In contrast, a proper hypothesis is easy to interpret and 
provides the most succinct representation. Specifically, a halfspace hypothesis would 
require only $O(d)$ time for evaluation and $O(d)$ storage space.
Even though it is known that $d^{\poly(1/\eps)}$ time is required 
for identifying a good hypothesis during training, prior to this work, it was 
not clear whether one can learn a succinct hypothesis that is more efficient at test time. 
}

The preceding discussion motivates the following natural question:
\begin{center}
{\em Is there an efficient {\em proper} agnostic learner for halfspaces under Gaussian marginals?}
\end{center}
The main result of this paper (Theorem~\ref{thm:proper-learner})
is the first agnostic proper learner for this problem whose complexity
qualitatively matches that of the known improper learner~\citep{KKMS:08}.

\paragraph{Faster Runtime via Approximate Learning.} 
In view of the known SQ lower bounds for our problem~\citep{DKZ20, GGK20, DKPZ21}, 
it is unlikely that the $d^{\poly(1/\eps)}$ runtime for agnostically learning halfspaces
can be improved, even under the Gaussian distribution.
A line of work~\citep{KLS09, ABL17, Daniely15, DKS18a, DKTZ20c} has focused on obtaining
faster learning algorithms with relaxed error guarantees.
Specifically,~\cite{ABL17} gave the first $\poly(d/\eps)$ time {\em constant-factor}
approximation algorithm -- i.e., an algorithm with misclassification error of
$C \cdot \opt+\eps$, for some universal constant $C>1$ -- for {\em homogeneous}
halfspaces under the Gaussian, and, more generally, under any isotropic log-concave distribution.
More recently,~\cite{Daniely15} obtained a polynomial time approximation scheme (PTAS), i.e.,
an algorithm with error $(1+\gamma) \cdot \opt+\eps$ and runtime $d^{\poly(1/\gamma)}/\poly(\eps)$,
under the uniform distribution on the sphere (and, effectively, under the Gaussian distribution).

Interestingly, the constant factor approximation algorithm of~\cite{ABL17} is proper. On the other hand,
the PTAS of~\cite{Daniely15} is inherently improper, in part because it relies on the combination of the localization method~\cite{ABL17} and the (improper) polynomial regression algorithm~\cite{KKMS:08}.
It is thus natural to ask the following question:
\begin{center}
{\em Is there a {\em proper} PTAS for agnostically learning
halfspaces under Gaussian marginals?}
\end{center}
As our second main contribution (Theorem~\ref{thm:proper-ptas}), we give such a proper PTAS
qualitatively matching the complexity of the known improper PTAS~\citep{Daniely15}.

\subsection{Our Contributions} \label{ssec:results}

In this paper, we initiate a systematic algorithmic investigation of proper learning
in the agnostic distribution-specific PAC model. Our main result is the first proper
agnostic learner for the class of halfspaces under the Gaussian distribution,
whose sample complexity and runtime qualitatively match the performance
of the previously known improper algorithm.

\begin{theorem}[Proper Agnostic Learning of Halfspaces]\label{thm:proper-learner}
Let $\D$ be a distribution on labeled examples $(\x, y) \in \R^d\times\{\pm 1\}$
whose $\x$-marginal is $\normal(\vec 0, \vec I)$.  There exists an algorithm that,
given $\eps, \delta>0$, and $N= d^{O(1/\eps^4)} \poly(1/\eps) \log(1/\delta) $ i.i.d.\ samples from $\D$,
the algorithm runs in time $\poly(N) + (1/\eps)^{O(1/\eps^6)}  \log(1/\delta)$, 
and computes a halfspace hypothesis $h$ such that,
with probability at least $1-\delta$, it holds
$\err_{0-1}^{\D}(h) \leq \opt+\eps$.
\end{theorem}

\noindent Theorem~\ref{thm:proper-learner} gives the first non-trivial
agnostic proper learner for the class of halfspaces under natural distributional assumptions.
The runtime of our algorithm is $d^{\poly(1/\eps)}$,
which {\em qualitatively} matches the complexity of the improper polynomial regression
algorithm~\citep{KKMS:08} and is known to be qualitatively best possible
in the SQ model~\citep{DKZ20, GGK20, DKPZ21}.

The analysis of~\cite{KKMS:08} established an upper bound of $d^{O(1/\eps^4)}$
on the complexity of polynomial regression for our setting. This bound was later improved to
$d^{O(1/\eps^2)}$, using optimal bounds on the underlying polynomial approximations~\citep{DKN10}.
Designing a proper learner that {\em quantitatively} matches this upper bound
is left as an interesting open question.

\medskip

Our second main contribution is the first proper polynomial-time approximation scheme
(PTAS) for the agnostic learning problem. In our context, a PTAS is an algorithm
that, for any $\gamma, \eps>0$,  runs in time $d^{\poly(1/\gamma)}/\poly(\eps)$
and outputs a hypothesis $h$ satisfying $\err_{0-1}^\D(h) \leq (1+ \gamma) \opt + \eps$.
The parameter $\gamma >0$ quantifies the approximation ratio of the algorithm.
Prior work~\citep{Daniely15} gave an improper PTAS for agnostically learning
{\em homogeneous} halfspaces, i.e., halfspaces whose separating hyperplane
goes through the origin. We give a proper algorithm for this problem.

\begin{theorem}[Proper PTAS for Agnostically Learning Halfspaces]\label{thm:proper-ptas}
Let $\D$ be a distribution on labeled examples $(\x, y) \in \R^d\times\{\pm 1\}$
whose $\x$-marginal is $\normal(\vec 0, \vec I)$. There exists an algorithm that,
given $\gamma, \eps, \delta>0$ and $N= d^{\poly(1/\gamma)} \poly(1/\eps) \log(1/\delta)$
i.i.d.\ samples from $\D$, runs in time $\poly(N,d)$, and computes a
halfspace $h$ such that, with probability $1-\delta$, it holds $\err_{0-1}^\D(h)\leq
(1+\gamma)\opt+\eps$, where $\opt$ is the optimal misclassification error of any
homogeneous halfspace.
\end{theorem}

\noindent Theorem~\ref{thm:proper-ptas} gives the first proper PTAS for
agnostically learning homogeneous halfspaces under any natural distributional assumptions
and qualitatively matches the complexity of the improper PTAS by~\cite{Daniely15}.
We note that the homogeneity assumption is needed for technical reasons
and is also required in the known improper learning algorithm.
Obtaining a PTAS for agnostically learning arbitrary halfspaces remains an
open problem (even for improper learners).

\begin{remark}[Extension to Other Non-Linear Activations] \label{rem:extensions}
{\em While the focus of the current paper is on the class of halfspaces,
our algorithmic techniques are sufficiently robust and naturally generalize
to other activation functions, i.e., functions of the form $f(\bx) = \sigma(\langle \bw, \bx \rangle)$,
where $\sigma: \R \to \R$ is a well-behaved activation function.
Specifically, in Appendix~\ref{sec:relus-ptas}, we use our methods
to develop the first proper agnostic learner for ReLU regression~\citep{DGKKS20}.}
\end{remark}

\paragraph{Broader Context} This work is the starting point of the broader research
direction of designing {\em proper} agnostic learners in the distribution-specific setting
for various expressive classes of Boolean functions. Here we make a first step in this direction
for the class of halfspaces under the Gaussian distribution. The polynomial regression
algorithm~\cite{KKMS:08} is an improper agnostic learner that has been showed to succeed
for broader classes of geometric functions, including degree-$d$ PTFs~\citep{DHK+:10, Kane11, DRST14, HKM14},
intersections of halfspaces~\citep{KKMS:08, KOS:08, Kane14}, and broader
families of convex sets~\citep{KOS:08}. An ambitious research goal is to develop a general
methodology that yields proper agnostic learners for these concept classes
under natural and broad distributional assumptions, 
matching the performance of polynomial regression.

\subsection{Overview of Techniques} \label{ssec:techniques}
In this section, we provide a detailed overview of our algorithmic and structural
ideas that lead to our proper learners.

\paragraph{Proper Agnostic Learning Algorithm}

The main idea behind our proper learning algorithm
is to start with a good improper hypothesis and compress it down to a halfspace,
while maintaining the same error guarantees. Our algorithm starts by computing
the low-degree polynomial $P$ that best approximates the labels in 
$L_2$-norm (Lemma~\ref{lem:polynomial-regr}).
We then take a two-step approach to identify a near-optimal halfspace.
First, by identifying the high-influence directions of $P$, we construct
a low-dimensional subspace of $\R^d$ and show that it
contains the normal vector to a near-optimal halfspace (Proposition~\ref{prop:structural}).
Then, we exhaustively search over vectors in this subspace
(through an appropriately fine cover) and output the one with minimum error.

The main technical challenge comes in identifying such a subspace
that is large enough to contain a good proper hypothesis,
but also small enough so that exhaustive searching is efficient.
To identify this subspace, we consider an appropriate matrix
(defined by the high-influence directions of the polynomial $P$)
and take the subspace defined by its large eigenvectors (Proposition~\ref{prop:structural}).
Exploiting the concentration guarantees of polynomials under the Gaussian distribution,
we show that the resulting subspace is small enough to enumerate over (Lemma~\ref{lem:dimension-bound}).

In more detail, we first find the polynomial $P(\x)$ of degree $k= O(1/\eps^4)$
that approximates the labels $y$ in the $L_2$ sense, that is,
minimizes $\E_{(\vec x, y)\sim \D}[ (y - P(\x))^2]$.
We then consider the influence of the polynomial $P$ along a direction $\vec u$,
$\mathrm{Inf}_{\vec u}(P) = \vec u^T \vec M \vec u$,
for the matrix $\vec M=\E_{\x \sim \D_{\x}}[\nabla P(\x)\nabla P(\x)^\top]$,
as a measure for how much the polynomial $P$ changes along the direction $\vec u$.
The key observation is that, along low-influence directions,
 the polynomial remains essentially constant and, as we show,
 the optimal halfspace must also be essentially constant as well.
 This allows us to prune down these directions and focus on a subspace of lower-dimension.
 Our main structural result (Proposition~\ref{prop:structural}) formalizes this intuition
 showing that the subspace $V$ of eigenvectors whose eigenvalues are larger
 than $\Theta(\eps^2)$ contains a normal vector $\vec w_V$ that
(together with an appropriate threshold) achieves error $\opt + \eps$.

Finally, while Proposition~\ref{prop:structural} establishes that we can remove
directions of low-influence, we need to argue that the number of relevant eigenvectors
is sufficiently small to simplify the problem. As we show in Lemma~\ref{lem:dimension-bound},
the dimension of the resulting subspace $V$ is $O(1/\eps^6)$,
and thus finding a good hypothesis in this subspace takes time independent
of the original dimension $d$. The key ingredient in bounding the dimension of $V$
is to use concentration of polynomials under the Gaussian distribution
to argue that the Frobenius norm of $\vec M$ is bounded,
and thus the number of eigenvectors with large eigenvalues is bounded.

\paragraph{Proper PTAS for Agnostic Learning}
Our algorithm for obtaining a proper PTAS works in the same framework as~\cite{Daniely15},
who gave a non-proper PTAS for homogeneous halfspaces
by combining the algorithm of~\cite{ABL17} with the $L_1$-polynomial
regression algorithm of~\cite{KKMS:08}.

Similarly to the algorithm of~\cite{Daniely15}, we start by learning
a halfspace (with normal vector) $\vec w_0$ with error $O(\opt)$, using any 
of the known constant factor approximations as a black-box~\citep{ABL17, DKS18a, DKTZ20c},
and then partition the space according to the distance to the halfspace $\vec w_0$.
Daniely's algorithm~\citep{Daniely15} is based on the observation
that points far from the true halfspace are accurately classified by the halfspace $\vec w_0$.
Thus, one can use the improper learner of~\cite{KKMS:08} to classify nearby points.

A simple adaptation of this idea would be to replace the improper algorithm of~\cite{KKMS:08}
with our new proper algorithm for agnostically learning halfspaces.
There are two main complications however.
First, the guarantees of our proper algorithm crucially rely on having Gaussian marginals, and therefore
we cannot readily apply it once we restrict our attention only to points around $\vec w_0$.
We deal with this issue by using a ``soft" localization technique introduced in~\cite{DKS18a}
to randomly partition points in two groups. In particular, we perform rejection sampling according
to a {judiciously chosen} weight function such that the distribution conditional on acceptance is still a Gaussian,
albeit with very small variance along the direction of $\vec w_0$, see Lemma~\ref{lem:rejection-sampling}.
By running our proper algorithm, we can obtain a halfspace $\vec w_1$ that is near-optimal under the conditional distribution.

The second obstacle is that while we can obtain two halfspaces ($\vec w_0$ and $\vec w_1$)
that each are near-optimal for their corresponding groups,
combining them into a \emph{single} halfspace that works well for the entire distribution is not immediate.
We remark that this is not an issue for the improper approximation scheme of \cite{Daniely15},
since an improper learner is allowed to output a different classifier for different subsets of $\R^d$.
To handle this issue, we additionally show that the halfspace $\vec w_1$ we obtain after localization
will in fact perform well overall.  In more detail, we show that the halfspace $\vec w_1$ cannot have
very large angle with $\vec w_0$ and also its bias is small, see Proposition~\ref{prop:error-acc}.
Given these closeness properties, we can then show that the halfspace $\vec w_1$
achieves the desired error guarantees over the entire distribution,
see Lemma~\ref{lem:localization}.
 \newcommand{\capfun}{\mathrm{cap}}

\section{Preliminaries}\label{sec:prelims}

We will use small
boldface characters for vectors and capital bold characters for matrices.  For $\bx \in \R^d$ and $i \in [d]$, $\bx_i$
denotes the $i$-th coordinate of $\bx$, and $\|\bx\|_2 \eqdef (\littlesum_{i=1}^d \bx_i^2)^{1/2}$ denotes the $\ell_2$-norm of $\bx$.
We will use $\bx \cdot \by $ for the inner product of $\bx, \by \in \R^d$ and $ \theta(\bx, \by)$ for the angle between $\bx, \by$.
We will use $\1_A$ to denote the characteristic function of the set $A$,
i.e., $\1_A(\x)= 1$ if $\x\in A$ and $\1_A(\x)= 0$ if $\x\notin A$.

Let $\vec e_i$ be the $i$-th standard basis vector in $\R^d$.
For $\x\in \R^d$ and $V\subseteq \R^d$, $\x_{V}$ denotes the projection of $\x$ onto the subspace $V$. Note that
in the special case where $V$ is spanned from one unit vector $\vec v$, then we simply write $\x_{\vec v}$
to denote $\vec v~(\vec x \cdot \vec v)$, i.e., the projection of $\vec x$ onto $\vec v$.
For a subspace $U\subset\R^d$, let $U^{\perp}$ be the orthogonal complement of
$U$. For a vector $\vec w\in\R^d$, we use $\vec w^\perp$ to denote the subspace spanned by vectors
orthogonal to $\vec w$, i.e., $\vec w^\perp=\{\vec u\in \R^d: \vec w \cdot \vec u=0\}$. For a matrix $\vec A\in \R^{d\times d}$,  $\tr(\vec A)$ denotes the trace of the matrix $\vec A$.

We use $\E_{x\sim \D}[x]$ for the expectation of the random variable $x$ according to the distribution $\D$ and
$\pr[\mathcal{E}]$ for the probability of event $\mathcal{E}$. For simplicity of notation, we may
omit the distribution when it is clear from the context. Let $\normal( \boldsymbol\mu, \vec \Sigma)$ denote the $d$-dimensional Gaussian distribution with mean $\boldsymbol\mu\in  \R^d$ and covariance $\vec \Sigma\in \R^{d\times d}$. \nnew{For $(\x,y)$ distributed according to $\D$, we denote $\D_\x$ to be the distribution of $\x$. For unit vector $\vec v\in \R^d$, we denote $\D_{\vec v}$ the distribution of $\x$ on the direction $\vec v$, i.e., the distribution of $\x_{\vec v}$.}

We use  ${\cal C}_{V}$  for the set of Linear Threshold Functions (LTFs) 
with normal vector contained in $V\subseteq\R^d$,
i.e., ${\cal C}_{V}=\{\sign(\vec v \cdot \x +t): \vec v\in V, \snorm{2}{\vec v}=1, t \in \R\}$; 
when $V=\R^d$, we simply write $\cal C$. 
\nnew{Moreover, we define ${\cal C}_0$ to be the set of unbiased LTFs, 
i.e., ${\cal C}_{0}=\{\sign(\vec v \cdot \x): \vec v\in \R^d, \snorm{2}{\vec v}=1\}$.}
\nnew{We denote by $\mathcal{P}_k$ the space of polynomials on $\R^d$ of degree at most $k$.}

 \section{Proper Agnostic Learning Algorithm}\label{sec:structural}

In this section, we present our proper agnostic learning algorithm for halfspaces,
establishing Theorem~\ref{thm:proper-learner}.
The pseudocode of our algorithm is given in Algorithm~\ref{alg:proper-learner}.

\begin{algorithm}[H]
\caption{Agnostic Proper Learning Halfspaces} \label{alg:proper-learner}
\begin{algorithmic}[1]
\Procedure{Agnostic-proper-learner}{$\eps, \delta, \D$}\\
\textbf{Input:} $\eps>0$, $\delta>0$ and sample access to distribution $\D$\\
\textbf{Output:} A hypothesis $h\in{\cal C}$ such as $\err_{0-1}^\D(h)\leq \min_{f\in {\cal C}}\err_{0-1}^\D(f)+\eps$ with probability $1-\delta$.
\State $k\gets C/\eps^4$, $\eta\gets \eps^2/C$. \Comment{$C$ is a sufficiently large constant}
\State Find $P(\x)$ such $\E_{(\x,y)\sim \D}[(y-P(\x))^2]\leq \min_{P'\in{\cal P}_k}\E_{(\x,y)\sim \D}[(y-P'(\x))^2]+O(\eps^3)$.
\State Let $\vec M=\E_{\x\sim \D_\x}[\nabla P(\x)\nabla P(\x)^\top]$.
\State Let $V$ be the subspace spanned by the eigenvectors of $\vec M$ whose eigenvalues are at least $\eta$.
\State Construct an $\eps$-cover ${\cal H}$ of LTF hypotheses with normal vectors in $V$ \Comment{}{see Fact~\ref{fct:cover}}.
\State Draw $\Theta(\frac{1}{\eps^2}\log(|{\cal H}|/\delta))$ i.i.d.\ samples
from $\D$ and construct the empirical distribution $\widehat\D$.
\State $h\gets \argmin_{h'\in {\cal H}} \err_{0-1}^{\widehat{\D}}(h')$\label{alg:emprical-outputs}
\State $\textbf{return } h$.
\EndProcedure
\end{algorithmic}
\end{algorithm}

\subsection{Analysis of Algorithm~\ref{alg:proper-learner}: Proof of Theorem~\ref{thm:proper-learner}}  \label{ssec:proper-analysis}

The main structural result that allows us to prove Theorem~\ref{thm:proper-learner}
is the following proposition, establishing the following: Given a multivariate polynomial $P$
of degree $\Theta(1/\eps^4)$ that correlates well with the labels, we can use its {\em high-influence
directions} to construct a subspace that contains a near-optimal halfspace.
Specifically, we show:

\begin{proposition}\label{prop:structural}
Let $C>0$ be a sufficiently large universal constant.
Fix any $\eps \in (0, 1]$ and set $k = C/\eps^4$.
Let $P(\x) \in \mathcal{P}_k$ be a degree-$k$ polynomial such that
$\E_{(\x,y)\sim \D}[(y-P(\x))^2]\leq \min_{P' \in \mathcal{P}_k}\E_{(\x,y)\sim \D}[(y-P'(\x))^2]+O(\eps^3)$.
Moreover, let $\vec M= \E_{\x\sim \D_\x}[\nabla P(\x)\nabla P(\x)^\top]$ and
$V$ be the subspace spanned by the eigenvectors of $\vec M$ with eigenvalues larger than $\eta$,
where $\eta=\eps^2/C$. Then, for any $f\in \cal C$, it holds
$\min_{\vec v\in V,t\in \R} \E_{(\x,y)\sim \D}[(f(\x)-\sign(\vec v\cdot \x+t))y]\leq \eps$.
\end{proposition}

The proof of Proposition~\ref{prop:structural} is the bulk of the technical work of this section and is
deferred to Section~\ref{ssec:prop-struct}.
In the body of this subsection, we show
how to use Proposition~\ref{prop:structural}
to establish Theorem~\ref{thm:proper-learner}.

The next lemma bounds from above the dimension of the subspace
spanned by the high-influence directions
of a degree-$k$ polynomial that minimizes the $L_2$-error with the labels $y$.

\begin{lemma}\label{lem:dimension-bound}
Fix $\eps>0$ and let $P(\x)$ be a degree-$k$ polynomial, with $k= O(1/\eps^4)$,
such that $\E_{(\x,y)\sim \D}[(y-P(\x))^2]\leq \min_{P'\in {\cal P}_k}\E_{(\x,y)\sim \D}[(y-P'(\x))^2]+O(\eps^3)$.
Let $\vec M= \E_{\x\sim \D_\x}[\nabla P(\x)\nabla P(\x)^\top]$ and
$V$ be the subspace spanned by the eigenvectors of $\vec M$ with eigenvalues larger than $\eta$.
Then the dimension of the subspace $V$ is $\dim (V) = O(k/\eta)$.
\end{lemma}
\begin{proof}
Let $P$ be a polynomial such that $\E_{(\x,y)\sim \D}[(y-P(\x))^2]\leq
\min_{P'\in {\cal P}_k}\E_{(\x,y)\sim \D}[(y-P'(\x))^2]+O(\eps^3)$ and let
$P^\ast=\argmin_{P'\in {\cal P}_k}\E_{(\x,y)\sim \D}[(y-P'(\x))^2]$.
First, we note that $\E_{(x,y)\sim \D}[(y-P^\ast(\x))^2] \leq \E_{(x,y)\sim \D}[(y- 0)^2] = 1$.
Using the inequality $(a+b)^2 \leq 2 a^2 + 2 b^2$, we have
\begin{align*}
\E_{(\x,y)\sim \D}[P(\x)^2]
&\leq
  2 \E_{(\x,y)\sim \D}[y^2]
+
2 \E_{(\x,y)\sim \D}[(y - P(x))^2]
\\
&\leq
2 +
2 \E_{(\x,y)\sim \D}[(y - P^\ast(x))^2]
+O(\eps^3)
\leq 5 \,.
\end{align*}

Let $V$ denote the subspace spanned by the eigenvectors of
$\vec M = \E_{\x\sim \D_\x}[\nabla P(\x) \nabla P(\x)^\top]$ with eigenvalues at
least $\eta$.  We will show that $m = \dim(V)=O(k/\eta)$.
We can write
\begin{align}\label{eq:bound-dim}
m \, \eta & \leq \tr\left( \E_{\x\sim \D_\x}[\nabla P(\x) \nabla P(\x)^\top]\right)
= \E_{\x\sim \D_\x} \left[\tr(\nabla P(\x) \nabla P(\x)^\top)\right]
= \E_{\x\sim\D_\x}\left[\snorm{2}{\nabla P(\x)}^2 \right]\;.
\end{align}
It is sufficient to show that $\E_{\x\sim\D_\x}[\snorm{2}{\nabla P(\x)}^2]=O(k)$.
By writing $P(\x)$ in the Hermite basis, from Fact~\ref{fct:hermite-gradient}, it holds
\begin{equation}\label{eq:bound-gradient}
\E_{\x\sim\D_\x}\left[\snorm{2}{ \nabla P(\x)}^2\right]
= \sum_{\alpha\in \N^d } |\alpha| c_\alpha^2 \leq k\E_{\x\sim \D_\x}[P(\x)^2] \leq 5 k \,.
\end{equation}
Combining Equations~\eqref{eq:bound-dim} and~\eqref{eq:bound-gradient},
we obtain that $m=O(k/\eta)$, and the proof is complete.
\end{proof}

For the proof of Theorem~\ref{thm:proper-learner}, we require a
standard result on \nnew{$L_2$}-polynomial regression required to compute
the polynomial of Proposition~\ref{prop:structural}.
The proof can be found on Appendix~\ref{ssec:polynomial-regr}.

\begin{lemma}[$L_2$-Polynomial Regression]\label{lem:polynomial-regr}
Let $\D$ be a distribution on $\R^d\times\{\pm 1\}$
whose $\x$-marginal is $\normal{(\vec 0, \vec I)}$. Let $k\in \Z_+$ and $\eps, \delta>0$.
There is an algorithm that draws $N=(d k)^{O(k)}\log(1/\delta)/\eps^2$
samples from $\D$, runs in time $\poly(N,d)$, and outputs a polynomial $P(\x)$ of
degree at most $k$ such that
$\E_{\x\sim \D}[(f(\x)-P(\x))^2]\leq \min_{P'\in {\cal P}_k} \E_{\x\sim \D}[(f(\x)-P'(\x))^2]+\eps$,
with probability $1-\delta$.
\end{lemma}

By running the $L_2$-regression algorithm of the above lemma,
we obtain a polynomial $P$ matching the requirements of our
dimension-reduction result (Proposition~\ref{prop:structural}).
To complete the proof of Theorem~\ref{thm:proper-learner}, we perform SVD on the influence
matrix $\vec M$ (see Proposition~\ref{prop:structural}), and then create a sufficiently fine cover
of the low-dimensional subspace $V$.

We require the following standard fact showing the existence of a small $\eps$-cover $\widetilde{V}$ of the set $V$,
i.e., a set $\widetilde{V}$ such that for any $\vec v\in V$ there exists $\tilde{\vec v}\in \widetilde{V}$
such that $\|\vec v -\tilde{\vec v}\|_2\leq \eps$.

\begin{fact}[see, e.g., Corollary 4.2.13 of \cite{Ver18}]\label{fct:cover}
For any $\eps>0$, there exists an \nnew{explicit} $\eps$-cover of the unit ball in $\R^k$,
with respect to the $\ell_2$-norm, of size $O(1/\eps)^k$.
\end{fact}

In order to create an effective discretization of the hypotheses, we need the following fact.

\begin{fact} \label{fct:approximation-facts}
Let $\D$ be a distribution on $\R^d\times\{\pm 1\}$ whose $\x$-marginal is $\normal(\vec 0, \vec I)$.
Let $\vec u,\vec v \in \R^d$ be unit vectors and $t_1,t_2\in \R$. Then the following holds:
\begin{enumerate}
\item $\E_{\x\sim \D_\x}[|\sign(\vec u \cdot \x + t_1)-\sign(\vec v \cdot \x + t_1)|]=O(\|\vec u -\vec v\|_2)$,
\item $\E_{\x\sim \D_\x}[|\sign(\vec u \cdot \x + t_1)-\sign(\vec u \cdot \x + t_2)|]= O(|t_1-t_2|) $ and,
\item if $|t|> \log(1/\eps)$ then, for any unit vector $\vec v\in \R^d$, $\E_{\x\sim \D_\x}[|\sign(\vec v \cdot \x + t)-\sign(t)|]=O(\eps)$.
\end{enumerate}
\end{fact}
\begin{proof}
The first statement is proved in \cite{DKS18a} (Lemma 4.2).
For the second statement, assuming without loss of generality that $t_2\geq t_1>0$, we note that
\[
  \E_{\x\sim \D_\x}[|\sign(\vec u \cdot \x + t_1)-\sign(\vec u \cdot \x + t_2)|]=
  \frac{2}{\sqrt{2}\pi} \int_{t_1}^{t_2}e^{-t^2/2}\d t
  = \pr_{t \sim \normal(0,1)}[t_1 \leq t \leq t_2]
  \;.
\]
Using the anti-concentration property of the one-dimensional Gaussian distribution, we have that
$ \pr_{t \sim \normal(0,1)}[t_1 \leq t \leq t_2] \leq O(t_2 - t_1)$, proving the claim.

For the third statement, note that if $t> \Omega(\sqrt{\log(1/\eps)})$, then
by the concentration properties of the Gaussian, we have that:
\[
  \E_{\x\sim \D_\x}[|\sign(\vec v \cdot \x + t)-\sign(t)|]
  \leq \pr_{\x\sim \D_\x}[|\vec v\cdot \x |\geq |t|]= O(\eps)\;.
\]
This completes the proof of Fact~\ref{fct:approximation-facts}.
\end{proof}

We are now ready to prove the main theorem of this section.

\begin{proof}[Proof of Theorem~\ref{thm:proper-learner}]
We first show that there is a set $\cal H$ of size $(1/\eps)^{O(1/\eps^6)}$
which contains tuples $(\vec u, t)$ with $\vec u\in \R^d$ and $t\in \R$,
such that
\[
  \pr_{(\x,y)\sim \D}[\sign(\vec u\cdot \x +t)\neq y]
  \leq \inf_{f\in \cal{C}}\pr_{(\x,y)\sim \D}[f(\x)\neq y] +\eps\;.
\]
First note we can assume that $1/\eps^6\leq d$, since otherwise
one can directly do a brute-force search over an $\eps$-cover
of the $d$-dimensional unit ball: we do not need to perform
our dimension-reduction process.  The runtime to perform this brute-force search
will be $(1/\eps)^{O(d)} \log(1/\delta)$ which, by the assumption that $1/\eps^6 > d$, is smaller
than $(1/\eps)^{O(1/\eps^6)} \log(1/\delta)$.

Let $f\in \cal C$ be such that the $\E_{(\x,y)\sim \D}[f(\x)y]$ is maximized and let $k=O(1/\eps^4)$.
By an application of Lemma~\ref{lem:polynomial-regr} for 
$N=(d/\eps)^{O(1/\eps^4)}\poly(1/\eps)\log(1/\delta)=d^{O(1/\eps^4)}\poly(1/\eps)\log(1/\delta)$, 
it follows that there exists a degree $O(1/\eps^4)$ polynomial $P(\x)$ such that
\[
  \E_{\x\sim \D}[(y-P(\x))^2]\leq \min_{P' \in {\cal P}_k} \E_{\x\sim \D}[(y-P'(\x))^2]+O(\eps^3)\;,
\]
with probability $1-\delta/2$.  Applying Proposition~\ref{prop:structural} to
the polynomial $P(\x)$, we get that the subspace $V$ spanned by the eigenvectors
of the matrix $\vec M= \E_{\x\sim \D_\x}[\nabla P(\x)\nabla P(\x)^\top]$ with
eigenvalues larger than $\eta=\Theta(1/\eps^2)$ contains a vector $\vec v\in V$, such that
\begin{equation}\label{eq:optimality-of-subspace}
\min_{t\in \R} \E_{(\x,y)\sim \D}[(f(\x)-\sign(\vec v\cdot \x+t))y]\leq \eps\;.
\end{equation}
Moreover, by Lemma~\ref{lem:dimension-bound}, the dimension of $V$ is $O(1/\eps^6)$.
Applying Fact~\ref{fct:cover}, we get that there exists an $\eps$-cover $\widetilde{V}$ of the set $V$
with respect the $\ell_2$-norm of size $(1/\eps)^{O(1/\eps^6)}$.
We show that there is an effective way to discretize the set of biases. From
Fact~\ref{fct:approximation-facts}, it is clear that the set
${\cal T}=\{\pm\eps, \pm 2\eps, \ldots, \pm O(\sqrt{\log(1/\eps)})\}$
is \nnew{an effective} cover of the parameter $t$.

It remains to show that the set $\cal H$ is an effective cover, where ${\cal H}=\widetilde{V}\times{\cal T}$.
We show that there exists a set of parameters $(\tilde{\vec v},\tilde t)\in \cal H$
which define a halfspace that correlates with the labels as well as the function $f$.
Fix the parameters $(\vec v,t)$ which minimize the Equation~\eqref{eq:optimality-of-subspace}.
Indeed, we have
\begin{align}\label{eq:bound-the-difference}
\E_{(\x,y)\sim \D} & [(f(  \x)-  \sign(\tilde{\vec v}\cdot  \x+\tilde{t}))y]\nonumber\\
& =\E_{(\x,y)\sim \D}[(f(\x)-\sign(\vec v\cdot \x+t))y] +
\E_{(\x,y)\sim \D}[(\sign(\vec v\cdot \x+t)-\sign(\tilde{\vec v}\cdot \x+\tilde{t}))y]\;.
\end{align}
We claim that there exists $(\tilde{\vec v},\tilde t)\in \cal H$ such that
\[ \E_{(\x,y)\sim \D}[(\sign(\vec v\cdot \x+t)-\sign(\tilde{\vec v}\cdot \x+\tilde{t}))y]= O(\eps)\;.\]
If $|t|>\sqrt{\log(1/\eps)}$, then from Fact~\ref{fct:approximation-facts},
the constant hypothesis gets $O(\eps)$ error, so we need to check the case $|t|\leq\sqrt{\log(1/\eps)}$.
Applying Fact~\ref{fct:approximation-facts}, we get that
\begin{equation}\label{eq:bias-bound}
\min_{(\tilde{\vec v},\tilde t)\in {\cal H}}\E_{(\x,y)\sim \D}[|\sign(\vec v\cdot \x+t)-\sign(\tilde{\vec v}\cdot \x+\tilde{t})|]=\min_{(\tilde{\vec v},\tilde t)\in {\cal H}}O(\|\vec v-\tilde{\vec v}\|_2+|t-\tilde{t}|)=O(\eps)\;.
\end{equation}
Thus, substituting Equation~\eqref{eq:bias-bound} to Equation~\eqref{eq:bound-the-difference}, we get
\begin{align}
\min_{(\tilde{\vec v},\tilde t)\in {\cal H}}\E_{(\x,y)\sim \D}[(f(\x)-\sign(\tilde{\vec v}\cdot \x+\tilde{t}))y]
\leq \E_{(\x,y)\sim \D}[(f(\x)-\sign(\vec v\cdot \x+t))y]  +O(\eps)= O(\eps)\;,
\end{align}
where in the last equality we used Equation~\eqref{eq:optimality-of-subspace}.
Using the fact that for a boolean function $g(\x)$ it holds $\E_{(\x,y)\sim \D}[g(\x)y]=1-2\pr_{(\x,y)\sim \D}[g(\x)\neq y]$,
we get
\[
  \min_{(\tilde{\vec v},\tilde t)\in {\cal H}}  \pr_{(\x,y)\sim \D}[\sign(\tilde{\vec u}\cdot \x
  +\tilde{t})\neq y]\leq \inf_{f\in \cal{C}}\pr_{(\x,y)\sim \D}[f(\x)\neq y] +O(\eps)\;.
\]
To complete the proof, we show that Step~\ref{alg:emprical-outputs} of Algorithm~\ref{alg:proper-learner} 
outputs a hypothesis close to the minimizer inside $\cal H$. From Hoeffding's inequality, 
it follows that $O(\frac{1}{\eps^2}\log({\cal H}/\delta))$
samples are sufficient to guarantee that the excess error of the chosen hypothesis is at most $\eps$
with probability at least $1-\delta/2$.
To bound the runtime of the algorithm, we note that $L_2$-regression has runtime
$d^{O(1/\eps^4)}\poly(1/\eps)\log(1/\delta)$ and exhaustive search 
over an $\eps$-cover takes time $(1/\eps)^{O(1/\eps^6)} \log(1/\delta)$.
Thus, the total runtime of our algorithm in the case
where $1/\eps^6 \leq d$ is
\[
  \Big( d^{O(1/\eps^4)} + (1/\eps)^{O(1/\eps^6)} \Big) \log(1/\delta) \,.
\]
This completes the proof of Theorem~\ref{thm:proper-learner}.
\end{proof}

\subsection{Proof of Proposition~\ref{prop:structural}} \label{ssec:prop-struct}

Suppose for the sake of contradiction that there exists a halfspace $f\in \cal C$ such that for every halfspace
$f'\in {\cal C}_{V}$, it holds
\begin{equation}\label{eq:contradiction}
\E_{(\x,y)\sim \D}[(f(\x)-f'(\x))y]\geq \eps\;.
\end{equation}
Our plan is to use the above fact in order to contradict the (approximate) optimality of the polynomial $P(\x)$.
To achieve this, we need to construct a polynomial $P''(\x)$ with error
\emph{strictly less than} $\min_{P' \in \mathcal{P}_k}\E_{(\x,y)\sim \D}[(y-P'(\x))^2]$.
In the following simple claim, we show that in order to construct such a polynomial $P''(\x)$,
one needs to find a polynomial $Q(\x)$ of degree at most $k$ that correlates well
with the difference $y - P(\x)$.

\begin{claim}\label{clm:contradiction}
It suffices to show that there exists a polynomial $Q(\x)$ of degree at most $k$ with $\E_{\x\sim \D_\x}[Q^2(\x)]\leq 9$
that $(\eps/4)$-correlates  with $(y- P(\x))$, i.e.,
$\E_{(\x,y)\sim \D}[Q(\x) (y-P(\x))]\geq \eps/4$.
\end{claim}
\begin{proof}
Given such a polynomial $Q(\x)$, we consider the polynomial $P''(\x) = P(\x)+ \zeta Q(\x)$,
for $\zeta= c\ \eps $ and $c$ a sufficiently small constant.
Observe that $P''(\x)$ has degree at most $k$ and
decreases the value of $\E_{(\x,y)\sim \D}[(y-P(\x))^2]$
by at least $\Omega(\eps^2)$, which contradicts the optimality of
$P(\x)$, i.e., that $P(\x)$ is $O(\eps^3)$-close to the polynomial that minimizes the
$L_2$-error with $y$.
\end{proof}

We now construct such a polynomial $Q(\x)$.
We can write that $f(\x)= \sign(\vec w \cdot \x + t) = \sign(\vec w_{V} \cdot \x + \vec w_{V^\perp} \cdot \x +t)$.
Note that $\vec w_{V^\perp} \neq \vec 0$, since otherwise we
would have $f \in {\cal C}_V$.
For simplicity, we denote $\vec \xi = \vec w_{V^\perp}/\|\vec w_{V^\perp}\|_2$.  Notice that
the direction $\vec \xi$ has low influence, since $\vec \xi \in V^\perp$.
Recall that by $\D_{\vec \xi}$ we denote the projection of $\D$ onto the (one-dimensional) subspace
spanned by $\vec \xi$.
We define $f_{V}(\x)=\E_{\vec z \sim \D_{\vec \xi }}[f(\vec z + \x_{V})]$ to be
a convex combination of halfspaces in ${\cal C}_V$.  In particular, $f_V(\x)$ is a smoothed version
of the halfspace $\sgn(\vec w_V \cdot \x + t)$ whose normal vector belongs in $V$.
Our argument consists of two main claims. In Lemma~\ref{lem:correlation-with-new-f},
we show that the function $f(\x) - f_V(\x)$ correlates non-trivially with $y - P(\x)$.
Then we show that we can approximate $f(\x)-f_V(\x)$ with a low-degree
polynomial $Q(\x)$ that maintains non-trivial correlation with $y - P(\x)$; see
Lemma~\ref{lem:polynomial-apx}. 
We start with the first lemma.
\begin{lemma}\label{lem:correlation-with-new-f}
It holds $\E_{\x\sim \D_\x}[(f(\x)-f_V(\x))(y-P(\x))] \geq \eps - 2 \sqrt{\eta} $.
\end{lemma}
\begin{proof}
We have that $f_{V}(\x)=\E_{\vec z \sim \D_{\vec \xi }}[f(\vec z + \x_{\vec \xi^\perp})]
  = \E_{\vec z \sim \D_{\vec \xi }}[\sign( \vec w_V \cdot \x_{V} + \vec w \cdot \vec z + t)]
$ and, since $f_{V}$ is a convex combination of halfspaces in ${\cal C}_V$,
from Equation~\eqref{eq:contradiction}, we see that $\E_{(\x,y)\sim \D}[(f(\x)-f_V(\x))y]\geq \eps$.
Thus, we have
\begin{align}
\label{eq:correlation}
\E_{(\x,y)\sim \D}[(f(\x) - f_V(\x)) (y-P(\x))]
 & = \E_{(\x,y)\sim \D}[(f(\x) - f_V(\x)) y]- \E_{\x\sim \D_\x}[(f(\x)-f_V(\x))P(\x)] \nonumber \\
 & \geq \eps - \E_{\x\sim \D_\x}[(f(\x)-f_V(\x))P(\x)] \,.
\end{align}
To deal with $\E_{\x\sim \D_\x}[(f(\x)-f_V(\x))P(\x)]$,
we first observe that for any function $g(\x)$ depending only on
the projection of $\x$ onto the subspace $\vec \xi^\perp$, i.e.,
such that
$g(\vec x) = g(\x_{\vec\xi^\perp})$,
we have
\[
  \E_{\x \sim \D_\x}[(f(\x)-f_V(\x))g(\x)] =
  \E_{\vec v \sim \D_{\vec \xi^{\perp}}} \left[\E_{\vec z \sim \D_{\vec \xi}} [f(\vec v + \vec z) - f_V(\vec v)] ~ g(\vec v)  \right] = 0 \,,
\]
since for every $\vec x \in \R^d$ it holds
$f_V(\vec x)
  = \E_{\vec z \sim \D_{\vec \xi}} [f(\vec x_{\vec \xi^\perp} + \vec z)]
  = \E_{\vec z \sim \D_{\vec \xi}} [f(\vec x_V + \vec z)]$.
Unfortunately, we cannot directly do the above trick because $P(\vec x)$ does not
depend only on $\vec x_{\vec \xi^\perp}$.
However, since $V$ contains the high-influence eigenvectors,
it holds that $P$ is almost a function of $\vec x_{\vec \xi^\perp}$.
In fact, we show that we can replace the polynomial $P$ by a different polynomial
of degree at most $k$ that only depends on the projection of $\x$ on $\vec \xi^\perp$.
Similarly to the definition of the ``smoothed" halfspace $f_V$, we
define $R(\x)=\E_{\vec z \sim \D_{\vec \xi}}[P(\x_{\vec \xi^\perp} + \vec z)]$.
We first prove that $R(\x)$ is close to $P(\x)$ in the $L_2$-sense.

\begin{claim}\label{clm:hermite-influence}
Let $R(\x)=\E_{\vec z \sim \D_{\vec \xi}}[P(\x_{\vec \xi^\perp} + \vec z)]$.
It holds $\E_{\x\sim \D_\x}[(P(\x) - R(\x))^2] \leq  \E_{\x\sim \D_\x}[(\nabla P(\x) \cdot \vec \xi)^2] $.
\end{claim}
\begin{proof}
We start by showing that without loss of generality we may assume that $\vec \xi =\vec e_1$.
Let $\vec U$ be an orthogonal matrix such that $\vec U\vec \xi = \vec e_1$.
Since $P(\x)$ is a polynomial, we can apply the orthogonal transformation $\vec U$ to $\x$
and then use the  Hermite basis to represent it, that is
$P(\x)=\sum_{\alpha\in \N^d} c_{\alpha}H_\alpha(\vec U^\top\x)$. Our objective is equivalent to
\[ \E_{\x\sim \D_\x}[(\nabla P(\x) \cdot \vec \xi)^2 -(P(\x) - R(\x))^2 ]\geq 0 \;.\]
By the change of variables $\x \mapsto \vec U \vec x$, we have
\[
  \E_{(\vec U\x)\sim \D_\x}\left[(\nabla_{\vec U\vec x} P(\vec U\x) \cdot (\vec U \vec \xi))^2 -(P(\vec U\x) - R(\vec U\x))^2 \right]\geq 0\;,
\]
where we used the chain rule for the gradient.
Observe that
\[R(\x)=\E_{\vec z \sim \D_{\vec \xi}}[P((\vec I-\vec \xi \vec\xi^\top)\x + \vec z)] =
  \E_{\vec z \sim \D_{\vec \xi}}\left[\sum_{\alpha\in \N^d} c_{\alpha}H_\alpha(\vec U^\top(\vec I-\vec \xi \vec\xi^\top)\x 
  + \vec U^\top \vec z)\right]\;,\]
thus
$R(\vec U\x)=\E_{\vec z \sim \D_{\vec \xi}}[\sum_{\alpha\in \N^d} c_{\alpha}H_\alpha(\vec U^\top(\vec I-\vec \xi \vec\xi^\top)\vec U\x 
+ \vec U^\top\vec z)]$.
Moreover, $P(\vec U\vec x)=\sum_{\alpha\in \N^d} c_{\alpha}H_\alpha(\vec U^\top\vec U\x)$.
Using the fact that $\vec U^\top \vec U=\vec I$ and $\vec U^\top \xi \xi^\top\vec U=\vec e_1 \vec e_1^\top$, 
it follows that without loss of generality, we may assume that $\vec \xi =\vec e_1$.

To keep notation simple, we write $P(\x) = \sum_{\alpha \in \N^d} c_\alpha H_\alpha(\x)$.
Note that
\begin{equation}\label{eq:hermite-expansion}
P(\x)-\E_{\x_1\sim \D_{\vec e_1}}[P(\x_{\vec \xi^\perp}+\x_1)] =
\sum_{\alpha\in \N^d} c_{\alpha}H_\alpha(\x)-\sum_{\alpha\in \N^d} c_{\alpha}\E_{\x_1\sim \D_{\x_1}}[H_\alpha(\x)]
= \sum_{\alpha\in {\cal S}} c_{\alpha}H_\alpha(\x) \;,
\end{equation}
where $\cal S$ contains all the tuples for which the first index is
non-zero, this follows from the fact that $\E_{\x\sim \D}[H_\alpha(\x)]=0$.
Applying Parseval's identity yields
\begin{align}\label{eq:hermite-coef}
\E_{\x\sim \D_\x}[(P(\x)-\E_{\x_1\sim \D_{\vec e_1}}[P(\x_{\vec \xi^\perp}+\x_1)])^2] & =\sum_{\alpha\in {\cal S}}c_\alpha^2 \,.
\end{align}
From Fact~\ref{fct:hermite-gradient}, we have
\begin{align}
\label{eq:influence-final}
\E_{\x\sim \D_\x}[(\nabla P(\x) \cdot \vec e_1)^2] = \sum_{\alpha\in \N^d} \alpha_1 c_\alpha^2 \geq\sum_{\alpha\in {\cal S}} c_\alpha^2 \;,
\end{align}
where we used that $\alpha_1 \geq 1$ on the set $\cal S$.
Combining~\eqref{eq:hermite-coef} and \eqref{eq:influence-final} completes the proof.
\end{proof}

Adding and subtracting $R(\x) = \E_{\vec z \sim \D_\vec \xi}[P(\x_{\vec \xi^\perp}+ \vec z)]$, we get
\begin{align*}
\E_{\x\sim \D_\x}[(f(\x) - f_V(\x)) P(\x)] =
\E_{\x\sim \D_\x}[(f(\x)-f_V(\x))(P(\x)-R(\x_{\vec \xi^\perp}))]
+\E_{\x\sim \D_\x}[(f(\x)-f_V(\x)) R(\x_{\vec \xi^\perp})]\;.
\end{align*}
The second term is equal to zero, from the fact that
$\E_{\vec z \sim \D_{\vec \xi}} [f(\vec z +\x_{\vec \xi^\perp})-f_V(\x_{\vec \xi^\perp})]=0$.
From the Cauchy-Schwartz inequality, we get
\begin{align}
\E_{\x\sim \D_\x}[(f(\x)-f_V(\x))(P(\x)-R(\x_{\vec \xi^\perp}))]
 & \leq \sqrt{\E_{\x\sim \D_\x}[(f(\x)-f_V(\x))^2]\E_{\x\sim \D_\x}[(P(\x) - R(\x_{\vec \xi^\perp}))^2]} \nonumber      \\
 & \leq 2\sqrt{\E_{\x\sim \D_\x}[(P(\x)- R(\x_{\vec \xi^\perp}))^2]}\leq 2\sqrt{\eta}\;,\label{eq:bound-of-polynomials}
\end{align}
where we used Claim~\ref{clm:hermite-influence}. Using Equation~\eqref{eq:correlation}, we get that
$\E_{(\x,y)\sim \D}[(f(\x) - f_V(\x)) (y-P(\x))]\geq \eps-2\sqrt{\eta}$.
which completes the proof of Lemma~\ref{lem:correlation-with-new-f}.
\end{proof}

Our final claim replaces $f-f_V$ by its polynomial approximation.
By Hermite concentration arguments, we can show that we can use
a polynomial $Q(\x)$ of degree  $O(1/\eps^4)$. More specifically, we show:

\begin{lemma}\label{lem:polynomial-apx}
There exists a  polynomial $Q(\x)$ of degree $O(1/\eps^4)$ such that
$\E_{\x\sim \D_\x}[Q(\x)(y - P(\x))] \geq \eps/2 - 2 \sqrt{\eta}$ and $\E_{\x\sim \D_\x}[Q^2(\x)]\leq 9$.
\end{lemma}
\begin{proof}
We will require the following result from \cite{KOS:08} which bounds the Hermite concentration of LTFs.
\begin{fact}[Theorem 15 of~\cite{KOS:08}]\label{fct:concetration-hermite}
Let $f\in\cal C$, and let $S$ be the Hermite expansion up to degree $k$
of $f$, i.e., $S(\x) = \sum_{|\alpha| \leq k} \hat{f}(\alpha) H_\alpha(\x)$. Then
$\E_{\x \sim \normal(\vec 0,\vec I)}[(S(\x)-f(\x))^2]=O(1/\sqrt k)$.
\end{fact}
For any polynomial $Q(\x)$, we have
\begin{align}
\E_{\x\sim \D_\x}[Q(\x)(y - P(\x))] & = \E_{\x\sim \D_\x}  [(Q(\x) + (f(\x)-f_V(\x)) - (f(\x) - f_V(\x)) )(y-P(\x)) ]\nonumber
\\   & \geq   \eps - 2 \sqrt{\eta} + \E_{\x\sim \D_\x}[(Q(\x) - (f(\x) - f_V(\x)) )(y-P(\x)) ]\label{eq:connective_equ}\;,
\end{align}
where we used Lemma~\ref{lem:correlation-with-new-f}.
By choosing $Q(\x)=S(\x)-\E_{\vec z \sim \D_{\vec \xi }}[S(\x_{\xi^\perp} +
    \vec z)]$, where we denote by $S(\x)$ the Hermite expansion of $f$ truncated up to
degree $k$, $S(\x) = \sum_{|\alpha| \leq k } \hat{f}(\alpha) H_\alpha(\x)$, we will show that
$
  \E_{\x\sim \D_\x}[(f(\x)-f_V(\x)-Q(\x))^2]=O(1/\sqrt{k})\;.
$
Using the elementary inequality $(a+b)^2\leq 2a^2 + 2b^2$, we get that
\[
  \E_{\x\sim \D_\x}[(f(\x)-f_V(\x)-Q(\x))^2]\leq 2\E_{\x\sim \D_\x}[(f(\x)-S(\x))^2] + 2\E_{\x\sim \D_\x}[(f_V(\x)-\E_{\vec z \sim \D_{\vec \xi }}[S(\x_{\xi^\perp})])^2]\;.
\]
Moreover, by Jensen's inequality, it holds that
\[
  \E_{\x\sim \D_\x}[(f_V(\x)-\E_{\vec z \sim \D_{\vec \xi }}[S(\x_{\xi^\perp}+\vec z)])^2]\leq \E_{\x\sim \D_\x}[(f(\x)-S(\x))^2]=O(1/\sqrt{k})\;,
\]
where in the last equality we used Fact~\ref{fct:concetration-hermite}.
Note that from the reverse triangle inequality, it holds that
\begin{equation}\label{eq:norm-bound}
\sqrt{\E_{\x\sim\D}[Q^2(\x)]}\leq \sqrt{  \E_{\x\sim \D_\x}[(f(\x)-f_V(\x))^2]} +O(1/k^{1/4})\leq 2+O(1/k^{1/4})\;.
\end{equation}
Choosing $k=O(1/\eps^4)$ and applying
Cauchy-Schwartz to the Equation~\eqref{eq:connective_equ}, we get
\begin{align*}
\E_{(\x,y)\sim \D}[Q(\x)(y - P(\x))] & \geq
\eps - 2 \sqrt{\eta} - \sqrt{\E_{\x\sim \D_\x}[(Q(\x) - (f(\x) - f_V(\x)))^2]\E_{(\x,y)\sim \D}[(y-P(\x))^2]}
\\ & \geq\eps/2 - 2 \sqrt{\eta} \;,
\end{align*}
where we used the fact that $\E_{(\x,y)\sim \D}[(y-P(\x))^2]\leq \E_{(\x,y)\sim \D}[(y-0)^2]\leq 1$;
the polynomial $P(\x)$ is closer to $y$ than the trivial polynomial $0$. For this choice of $k$, Equation~\eqref{eq:norm-bound} gives $\E_{\x\sim\D}[Q^2(\x)]\leq 9$.
This completes the proof of Lemma~\ref{lem:polynomial-apx}.
\end{proof}
By choosing $\eta=\eps^2/64$, Lemma~\ref{lem:polynomial-apx} contradicts our
assumption that $P(\x)$ is $O(\eps^3)$-close to the polynomial $P'(\x)$ that
minimizes the $\E_{(\x,y)\sim \D}[(y-P'(\x))^2]$, see Claim~\ref{clm:contradiction}. This completes the proof
of Proposition~\ref{prop:structural}.

 \section{Agnostic Proper PTAS for Homogeneous Halfspaces}\label{sec:ptas}

In this section, we provide a proper PTAS for agnostically learning homogeneous halfspaces,
thereby establishing Theorem~\ref{thm:proper-ptas}.
Concretely, let $f(\x) = \sgn(\vec w^* \cdot \x)$ be an optimal halfspace, i.e.,
$\opt=\min_{h\in {{\cal C}_0}}\err_{0-1}^\D(h) = \err_{0-1}^\D(f) $.
For any $\gamma, \eps \in (0,1)$ our algorithm computes a halfspace $h(\x)$
such that $\err_{0-1}^\D(h) \leq (1+ \gamma) \opt + \eps $.

The pseudocode of our algorithm is given in Algorithm~\ref{alg:proper-ptas}.

\begin{algorithm}[H]
  \caption{Agnostic Proper PTAS for Homogeneous Halfspaces} \label{alg:proper-ptas}
  \begin{algorithmic}[1]
    \Procedure{Agnostic-Proper-PTAS}{$\gamma, \eps, \delta, \D$} \Comment{$C$, $C'$ are absolute constants}\\
    \textbf{Input:} $\gamma>0$, $\eps, \delta>0$, and  distribution $\D$\\
    \textbf{Output:} A hypothesis $h\in{\cal C}$ such that $\err_{0-1}^\D(h)\leq (1+\gamma)\opt+\eps$ with probability $1-\delta$
    \State $\sigma\gets C' \, \opt/\gamma$
    \State Let $\vec w_0$ be the normal vector of the homogeneous halfspace computed using Lemma~\ref{lem:initialization}
    \State \textbf{If} $\eps> C \, \opt$:
    \State \qquad \textbf{return} $h_0(\x)=\sgn(\vec w_0 \cdot \x)$
    \State Let $\D_A$ be the distribution $\D$ after applying rejection sampling with $\vec w_0$ and $\sigma$, see Fact~\ref{lem:rejection-sampling}
    \State Run Algorithm~\ref{alg:proper-learner} on $\D_A$ with accuracy $\Theta(\gamma^2)$ and confidence $\delta$ to get $(\vec w, t)$
    \State  \textbf{return} $h(\x)=\sgn(\vec w \cdot \x+t)$
    \EndProcedure
  \end{algorithmic}
\end{algorithm}

\subsection{Analysis of Algorithm~\ref{alg:proper-ptas}: Proof of Theorem~\ref{thm:proper-ptas}} \label{ssec:ptas-analysis}

The following lemma provides us with an efficient algorithm that learns a halfspace within error $O(\opt)$
in polynomial time.  This halfspace serves as the initialization of Algorithm~\ref{alg:proper-ptas}: we will use
it to perform localization around it.

\begin{lemma}[\cite{ ABL17,DKTZ20c}]\label{lem:initialization}
Let $\D$ be a distribution on $\R^d\times\{\pm 1\}$ whose $\x$-marginal is $\normal(\vec 0, \vec I)$.
There is an algorithm that draws $N=O((d/\eps^4)\log(1/\delta))$ samples from $\D$, runs in time $\poly(N,d)$,
and outputs a hypothesis $h \in \cal C$ such that, with probability at least $1-\delta$,
we have $\err_{0-1}^\D(h)\leq O(\opt)+\eps$.
\end{lemma}

The following lemma provides a ``soft" localization procedure.  Instead
of performing rejection sampling inside a band around $\vec w_0$, i.e.,
$|\x \cdot \vec w_0| < \sigma$, we perform rejection sampling with
weight $ e^{-\vec w_0 \cdot \x (\sigma^{-2}-1)}$: samples that are far
from the halfspace $\vec w_0$ are accepted with very  small probability.
Using this rejection sampling process, we get that the distribution conditional on
acceptance is a normal distribution. 
This allows us to use our proper learning algorithm of Section~\ref{sec:structural}.

\begin{lemma}[Lemma 4.7 of \cite{DKS18a}]\label{lem:rejection-sampling}
Let $\vec w_0\in \R^d$ be a unit vector and let $\D$ be a distribution on
$\R^d\times \{\pm 1\}$ whose $\x$-marginal is $\normal(\vec 0,\vec I)$. Fix
$\sigma \in (0,1)$ and define the distribution $\D_{A}$ as follows: draw a
sample $(\x, y)$ from $\D$ and accept it with probability
$e^{-(\vec w_0 \cdot \x )^2(\sigma^{-2}-1)/2}$.
$\D_A$ is the distribution of $(\x,y)$ conditional on acceptance.
The $\x$-marginal of $\D_{A}$ is $\normal(\vec 0, \vec \Sigma)$,
where $\vec \Sigma = \vec I - (1-\sigma^2) \vec w_0 \vec w_0^\top$, and the
probability that some point will be accepted is $\sigma$.
\end{lemma}

The main technical tool of this section is the following proposition.  It shows
that, if a halfspace performs reasonably well with respect to the ``localized"
distribution $\normal(\vec 0, \Sigma)$ of Lemma~\ref{lem:rejection-sampling},
then it cannot be very biased or have very large angle with the initial guess
$\vec w_0$.  This allows us to prove that the halfspace that we find using the
``localized" distribution will perform well over the initial Gaussian,
$\normal(\vec 0, \vec I)$.

\begin{proposition}\label{prop:error-acc}
Let $\vec w_0 \in \R^d$ be a unit vector and let $\alpha, \gamma \in (0,1/4]$.
Let $\D_A$ be defined as in Fact~\ref{lem:rejection-sampling},
i.e., its $\x$-marginal is $\normal(\vec 0, \vec \Sigma)$,
where $\vec \Sigma = \vec I - (1-\sigma^2) \vec w_0 \vec w_0^T$ for some $\sigma \in (0, \cos(\pi \alpha))$.
Moreover, assume that $\pr_{(\x, y) \sim \D}[ \sgn(\vec w_0 \cdot \x) \neq y ] \leq \alpha/3$.
There exists an algorithm that runs in time $d^{\poly(1/(\gamma\alpha))} \log(1/\delta)$ and
with probability at least $1 - \delta$ returns a halfspace $h(\x) = \sgn(\vec w \cdot \x + t)$ such that
$|t| = O(\sigma \alpha)$, $\theta(\vec w, \vec w_0) = O(\sigma \alpha)$.
Moreover, it holds
\[
    \pr_{(\x, y) \sim \D_A}[ h(\x) \neq y ] \leq
    \min_{\bar{h}\in {\cal C}} \pr_{(\x, y) \sim \D_A}[ \bar h(\x) \neq y ] + \alpha \gamma
    \,.
  \]
\end{proposition}

\noindent The proof of Proposition is quite technical and is deferred 
to Section~\ref{ssec:prop-error-acc}.

The following lemma is similar to the localization lemma (Lemma 2.1) 
given in~\cite{Daniely15}. We need to adapt it to work in our setting, 
where we use a soft localization procedure (see Lemma~\ref{lem:rejection-sampling}), 
as opposed to a hard one.

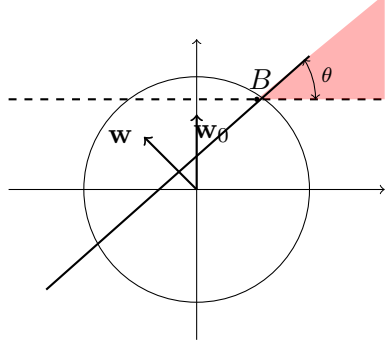
\begin{figure}

  \begin{center}

    \begin{tikzpicture}[scale=1]
      \coordinate (start) at (2.5,1.2);
      \coordinate (center) at (2.9/5,1.2);
      \coordinate (end) at (2.5-1,4.44/5 * 2.5 -2.22/5);

      \draw[black,dashed, thick](-2.5,1.2) -- (2.5,1.2);
      \draw[fill=red, opacity=0.3,draw=none]  (0.8,1.2)--(2.5,1.2)--(2.5,2.6)--cycle;
\draw[->] (-2.5,0) -- (2.5,0) node[anchor=north west,black] {};
      \draw[->] (0,-2) -- (0,2) node[anchor=south east] {};
      \draw[thick,->] (0,0) -- (-0.7,0.7) node[anchor= south east,below,left=0.1mm] {$\vec w$};
      \draw[black,thick] (-2,-1.332) -- (2.5-1,4.44/5 * 2.5 -2.22/5);
      \draw[thick ,->] (0,0) -- (0,1) node[right=2mm,below] {$\vec w_0$};
\draw (0.85,1.2)node[above] {$B$};
      \draw (0,0) circle [radius=1.5];
      \fill (0.8,1.2) circle [radius=1pt];
\pic [draw, <->, angle radius=10mm, angle
        eccentricity=1.2, "$\scriptstyle\theta$"] {angle = start--center--end};
    \end{tikzpicture}
  \end{center}
  \caption{Localization technique, Lemma~\ref{lem:localization}}  \label{fig:localization}
\end{figure}

\begin{lemma}[Gaussian Localization]\label{lem:localization} 
Let $R(\x)$ be the event that the sample $\x$ is rejected 
from the rejection sampling procedure of Lemma~\ref{lem:rejection-sampling} 
with vector $\vec w_0$ and $\sigma=\Theta\left(\frac{\opt}{\alpha}\right)$.
Let $h(\x)=\sign(\vec w_0 \cdot \x)$, $h'(\x)=\sign(\vec w \cdot \x +t)$ be halfspaces
with $t=O(\sigma\alpha)$ and $\theta(\vec w_0,\vec w)=O(\sigma \alpha)$.
Then, $\pr_{\x \sim \D_\x}[ h(\bx) \neq h'(\x),R(\x)] =O(\alpha\, \opt)$.
\end{lemma}

\begin{proof}
Let $\theta=\theta(\vec w_0, \vec w)$ and fix $r=\Theta(1/\alpha^{1/3})\max(1,t/\sin\theta)$.
We define $B= r\sin \theta (\sqrt{1-t^2/r^2} -t/r \frac{\cos\theta}{\sin\theta})$. 
Observe that in order for $B>0$ we need $r\geq t/\sin\theta$, which is true by our assumptions.
Also note that $B=\Theta(r\sin\theta)$. We have that
\begin{align*}
\pr_{\x \sim \D_\x}[ h(\bx) & \neq h'(\x),R(\x)]=\\
& \pr_{\x \sim \D_\x}[ h(\bx) \neq h'(\x), |\vec w_0 \cdot \x|\geq B, R(\x)]
+\pr_{\x \sim \D_\x}[ h(\bx) \neq h'(\x), |\vec w_0 \cdot \x|< B, R(\x)]\;.
  \end{align*}
We first bound from above the term $ \pr_{\x \sim \D_\x}[ h(\bx) \neq h'(\x), |\vec w_0 \cdot \x|< B, R(\x)]$.
It holds that
\begin{align}\label{eq:bound-first-term-tails}
\pr_{\x \sim \D_\x}[ h(\bx) \neq h'(\x), |\vec w_0 \cdot \x|< B, R(\x)] 
& \leq \pr_{\x \sim \D_\x}[ |\vec w_0 \cdot \x|< B, R(\x)] \nonumber  \\
& = \erf(B/\sqrt{2})-\sigma \erf(B/(\sigma\sqrt{2})) \nonumber \\
&=O(B^3/\sigma^2)=O(\alpha\, \opt)\;,
\end{align}
where $\erf$ is the error function and in the last equality we used the error of the Taylor expansion of degree-$2$.
In order to bound the second term, we define $V$ to be the subspace spanned by the vectors $\vec w_0, \vec w$. 
Let $\x',x''$ be the solutions of the system of equations $\{\vec w \cdot \x +t =0, \snorm{2}{\x_V}=r^2\}$; 
observe that $\min(|\x'\cdot \vec w_0|,|\x''\cdot \vec w_0|) =B$, thus, 
if $\snorm{2}{\x_V}\leq r$ and $|\vec w_0 \cdot \x|\geq B$, then  $h(\x)=h'(\x)$ (see Figure~\ref{fig:localization}), 
thus
\begin{align}\label{eq:bound-second-term-tails}
\pr_{\x \sim \D_\x}[ h(\bx) \neq h'(\x), |\vec w_0 \cdot \x|\geq B,R(\x)] 
& \leq \pr_{\x \sim \D_\x}[ h(\bx) \neq h'(\x), |\vec w_0 \cdot \x|\geq B, R(\x)]\nonumber \\
&\leq C\theta e^{-r}=O(\alpha\, \opt)\;.
\end{align}
Combining Equations~\eqref{eq:bound-first-term-tails}, \eqref{eq:bound-second-term-tails}, we get
$\pr_{\x \sim \D_\x}[ h(\bx) \neq h'(\x), R(\x)] =O(\alpha\, \opt)$.
\end{proof}

We are now ready to prove Theorem~\ref{thm:proper-ptas}.

\begin{proof}[Proof of Theorem~\ref{thm:proper-ptas}]
Our analysis follows the cases of Algorithm~\ref{alg:proper-ptas}.
Initially, Algorithm~\ref{alg:proper-ptas} computes 
$h_0=\sgn(\vec w_0 \cdot \x)$, using the algorithm of Lemma~\ref{lem:initialization}.
From Lemma~\ref{lem:initialization}, we have that for this halfspace it holds,
with probability at least $1-\delta/2$, that $\err_{0-1}^\D(h_0)= C \, \opt+\eps$
for some absolute constant $C>1$.  The runtime of this step is $\poly(d, 1/\eps) \log(1/\delta)$.
Therefore, when $\eps > 2 C \, \opt$, we directly get that $\err_{0-1}^\D(h_0) \leq \opt+\eps$.

For the case when $\eps \leq 2 C \, \opt$, Algorithm~\ref{alg:proper-ptas} returns 
a halfspace $h(\x)=\sgn(\vec w \cdot \x + t)$, where $\vec w\in \R^d$ and $t \in \R$. 
We show that for this hypothesis $h$ it holds $\err_{0-1}^\D(h)\leq (1+\gamma)\opt+\eps$.
Let $\D_A$ be the distribution conditional on acceptance 
(see Lemma~\ref{lem:rejection-sampling}) with parameters $\vec w_0$ 
(the normal vector of the halfspace $h_0$) and $\sigma=\Theta(\opt/\alpha)$, 
for some parameter $\alpha=\Theta(\gamma)$ and $R(\x)$ (resp. $A(\x)$) 
be the event that the sample $\x$ is rejected (resp. accepted). 
The error of the hypothesis $h$ is
\begin{equation}\label{eq:final-bound}
\pr_{(\x,y)\sim \D}[h(\x)\neq y]= \pr_{(\x,y)\sim \D}[h(\x)\neq y,A(\x)] +
\pr_{(\x,y)\sim \D}[h(\x)\neq y,R(\x)]\;.
\end{equation}
Let us first bound the 
$\pr_{(\x,y)\sim \D}[h(\x)\neq y,A(\x)]= \pr_{(\x, y) \sim \D_A}[ h(\x) \neq y ]\pr_{\x\sim \D_\x}[A(\x)]$.
From Proposition~\ref{prop:error-acc}, we have that with sample complexity and 
runtime $d^{\poly(1/\gamma)} \log(1/\delta)$ we get that, with probability at least $1-\delta/2$, 
it holds
  \[
    \pr_{(\x, y) \sim \D_A}[ h(\x) \neq y ] \leq
    \min_{\bar{h}\in {\cal C}} \pr_{(\x, y) \sim \D_A}[ \bar h(\x) \neq y ] + \alpha \gamma 
    \leq \pr_{(\x, y) \sim \D_A}[ f(\x) \neq y ] + \alpha \gamma \,,
  \]
and $|t|=O(\sigma \alpha)$, $\theta(\vec w, \vec w_0)=O(\sigma \alpha)$. From Lemma~\ref{lem:rejection-sampling}, it holds $\pr_{\x\sim \D_\x}[A(\x)]=\sigma$, 
thus
\begin{equation}\label{eq:first-term}
\pr_{(\x,y)\sim \D}[h(\x)\neq y,A(\x)]\leq \pr_{(\x, y) \sim \D_A}[ f(\x) \neq y,A(\x) ] 
+ \alpha \gamma \sigma\;.
\end{equation}
To bound $\pr_{(\x,y)\sim \D}[h(\x)\neq y,R(\x)]$, observe that 
$\theta(\vec w_0,\vec w^*)=O(\opt)$, because 
$\pr_{(\x,y)\sim \D}[h_0(\x)\neq f(\x)]= \theta(\vec w_0, \vec w^*)/ \pi = O(\opt)$. Thus, 
with two applications of Lemma~\ref{lem:localization}, we get 
$\pr_{(\x,y)\sim \D}[f(\x)\neq h_0(\x),R(\x)]=O(\alpha\opt)$ and 
$\pr_{(\x,y)\sim \D}[h(\x)\neq h_0(\x),R(\x)]=O(\alpha\opt)$.
Using the triangle inequality, we get
\begin{align}\label{eq:second-term}
\nonumber  & \pr_{(\x,y)\sim \D}[h(\x)\neq y,R(\x)]  \leq \pr_{(\x,y)\sim \D}[f(\x)\neq y,R(\x)]
    + \pr_{(\x,y)\sim \D}[h(\x)\neq h_0(\x),R(\x)] \\
    &+\pr_{(\x,y)\sim \D}[f(\x)\neq h_0(\x),R(\x)]=\pr_{(\x,y)\sim \D}[f(\x)\neq y,R(\x)] +O(\alpha \opt)\;.
\end{align}
Substituting Equations~\eqref{eq:first-term} and \eqref{eq:second-term} into Equation~\eqref{eq:final-bound},
we get
\begin{align*}
\pr_{(\x,y)\sim \D}[h(\x)\neq y] 
& = \pr_{(\x, y) \sim \D_A}[ f(\x) \neq y,A(\x) ]+ \pr_{(\x,y)\sim \D}[f(\x)\neq y,R(\x)] + 
O(\alpha \opt +\alpha \gamma \sigma) \\
&= \opt + O((\alpha +\gamma)) \opt=(1+O(\gamma))\opt\;,
\end{align*}
where we used the fact that $\alpha = \Theta(\gamma)$.
Combining the above cases, we obtain that
$\pr_{(\x,y)\sim \D}[h(\x)\neq y]  \leq  (1 + O(\gamma)) \opt + \eps$.
Combining the runtime of the above two steps, we obtain that the total runtime 
of our algorithm is $d^{\poly(1/\gamma)} \poly(1/\eps) \log(1/\delta)$.
\end{proof}

\subsection{Proof of Proposition~\ref{prop:error-acc}} \label{ssec:prop-error-acc}

  We first make the $\x$-marginal isotropic by multiplying samples with $\vec \Sigma^{-1/2}$
  and then use the proper learning algorithm of Theorem~\ref{thm:proper-learner}
  with target accuracy $\eps = \alpha \gamma$. The sample complexity and runtime
  are thus $d^{\poly(1/(\alpha \gamma))}$.  From the guarantee of Theorem~\ref{thm:proper-learner},
  we immediately obtain that 
  $\pr_{(\x, y) \sim \D_A}[ h(\x) \neq y ] \leq
    \min_{\bar{h}\in {\cal C}} \pr_{(\x, y) \sim \D_A}[ \bar h(\vec \x) \neq y ] + \alpha \gamma/3$.
  It now remains to bound the bias $t$ and the angle $\theta(\vec w, \vec w_0)$
  of the returned halfspace $h(\x) = \sgn(\vec w \cdot \x + t)$.
  Using our assumption for the misclassification error of $\vec w_0$ with respect to
  $\D_A$, we obtain that 
  $\min_{\bar{h} \in \cal C} \pr_{(\x,y) \sim \D_A}[\sgn(\vec w_0 \cdot \x) \neq y ]
    \leq \pr_{(\x,y) \sim \D_A}[\sgn(\vec w_0 \cdot \x) \neq y ]  \leq \alpha$.
  From the triangle inequality, we obtain
  \[
    \pr_{(\x,y) \sim \D_A}[h(\x) \neq \sgn(\vec w_0 \cdot \x)] \leq
    \pr_{(\x,y) \sim \D_A}[h(\x) \neq y ] +
    \pr_{(\x,y) \sim \D_A}[\sgn(\vec w_0 \cdot \x) \neq y ]
    \leq \alpha\;.
  \]
  Therefore, we have
  \[
    \pr_{(\x,y) \sim \D_A}[ h(\x) \neq \sgn(\vec w_0 \cdot \x) ]
    =
    \pr_{\x \sim \normal(\vec 0,\vec I)}[h(\vec \Sigma^{1/2} \x) \neq \sgn((\vec \Sigma^{1/2} \vec w_0) \cdot \x)] \,.
  \]
  If the halfspace $h(\x)=\sgn(\vec w \cdot \x + t)$ has zero bias, i.e., $t=0$,
  we have that by the spherical symmetry of the Gaussian distribution it holds
  \[ \pr_{\x \sim \normal(\vec 0,\vec I)}[h(\vec \Sigma^{1/2} \x) \neq \sgn((\vec \Sigma^{1/2} \vec w_0) \cdot \x)]
    = \theta( \vec \Sigma^{1/2} \vec w, \vec \Sigma^{1/2} \vec w_0)/\pi \;.\]
  Unfortunately, the same is not true when one of the halfspaces has non-zero bias.  However,
  we can prove that the angle
  $\theta( \vec \Sigma^{1/2} \vec w, \vec \Sigma^{1/2} \vec w_0)/\pi$
  is still a \emph{lower bound} on the probability of disagreement, i.e.,
  $\pr_{\x \sim \normal(\vec 0,\vec I)}[h(\vec \Sigma^{1/2} \x) \neq \sgn((\vec \Sigma^{1/2} \vec w_0) \cdot \x)]$.

Formally, we prove the following claim showing that when one of the halfspaces is homogeneous,
the probability mass of the disagreement region is at least a constant multiple of the angle between the normal vectors.  
The proof follows from the observation that we can always minimize the disagreement probability between 
a homogeneous and an arbitrary halfspace by centering the Gaussian exactly at their intersection point.  
We provide the detailed proof in Appendix~\ref{app:claim1}.
\begin{claim} \label{clm:angle-lower-bound}
For  $\vec v, \vec u \in \R^d, t \in \R$ define the halfspaces $h_0(\x) = \sgn(\vec u \cdot \x)$, $h_1(\x) = \sgn(\vec v \cdot \x+ t)$.
It holds $\pr_{\x \sim \normal(\vec 0,\vec I)}[h_1(\x) \neq h_0(\x)] \geq \theta( \vec u, \vec v)/\pi$.
\end{claim}

Using Claim~\ref{clm:angle-lower-bound} and the fact that
$\pr_{(\x,y) \sim \D_A}[h(\x) \neq \sgn(\vec w_0 \cdot \x)] \leq \alpha$ that we showed above, we have that
$\theta( \vec \Sigma^{1/2} \vec w, \vec \Sigma^{1/2} \vec w_0) \leq \pi \alpha$.
We have
  \[
    \cos(\theta( \vec \Sigma^{1/2} \vec w, \vec \Sigma^{1/2} \vec w_0))
    = \frac{\vec w \cdot(\vec \Sigma \vec w_0)}{\sqrt{\vec w \cdot(\vec \Sigma \vec w)}
      {\sqrt{\vec w_0 \cdot(\vec \Sigma \vec w_0)}}}
    = \frac{\sigma~\vec w \cdot \vec w_0 }{ \sqrt{1- (1- \sigma^2) (\vec w \cdot \vec w_0)^2}} \,.
  \]
  Since $\theta( \vec \Sigma^{1/2} \vec w, \vec \Sigma^{1/2} \vec w_0) \leq \pi \alpha$
  and cosine is a decreasing function in $[0,\pi]$, we obtain that
  \[
    \sigma~\vec w \cdot \vec w_0
    \geq
    \cos(\pi \alpha)  \sqrt{1- (1- \sigma^2) (\vec w \cdot \vec w_0)^2} \,.
  \]
  Solving this quadratic inequality with respect to $\vec w \cdot \vec w_0$, we obtain
\begin{equation} \label{eq:correlation-lower-bound}
\vec w \cdot \vec w_0 \geq \sqrt{\frac{1}{1 + \sigma^2 (\frac1{\cos^2(\pi \alpha)} - 1)}}
=\sqrt{\frac{1}{1 + \sigma^2 \tan^2(\pi \alpha)}}\,.
\end{equation}
Using the inequality $\cos^{-1}( \sqrt{1/(1+x)}) \leq \sqrt{x}$ that holds for every $x\geq 0$, we obtain
that the angle $\theta(\vec w, \vec w_0) \leq \sigma \tan(\pi \alpha)$.
Using the elementary inequality $\tan(\pi x)  \leq 4 x $ that holds for all $x\in [0,1/4]$, we can
further simplify the bound for the angle to $\theta(\vec w, \vec w_0) \leq 4 \sigma \alpha = O(\sigma \alpha)$.

  We next bound the bias of the returned halfspace $h$.  Now that we know that the angle between
  the vectors $\vec w_0, \vec w$ is small, we can use the following lower bound on the disagreement
  between two halfspaces to get that the bias cannot be too large.  We provide the proof of the
  following claim in Appendix~\ref{app:claim1}.
  \begin{claim}
    \label{clm:hermite-lower-bound}
    Let $h_1(\x) = \sgn(\vec u \cdot \x + t_1)$, $h_2(\x) = \sgn(\vec v \cdot \x + t_2)$
    be two halfspaces.
    Let $r_1 = t_1/\snorm{2}{\vec u \vec \Sigma^{1/2}}$, $r_2 = t_2/\snorm{2}{\vec v \vec \Sigma^{1/2}}$.
It holds
    \[
      \pr_{\x \sim \normal(\vec 0, \vec \Sigma)}[ h_1(\x) \neq h_2(\x)]
      \geq
      \pr_{r \sim \normal(0,1)}[\min(r_1,r_2) \leq r \leq \max(r_1,r_2)]
      \,.
\]
  \end{claim}
From Claim~\ref{clm:hermite-lower-bound} and the fact
  that
  $\pr_{(\x,y) \sim \D_A}[ h(\x) \neq \sgn(\vec w_0 \cdot \x) ] \leq \alpha$,
  we obtain that
  $\pr_{r \sim \normal(0,1)}[ 0 \leq r \leq |t_1|/\snorm{2}{\vec w \vec \Sigma^{1/2}}]
    \leq  \alpha $.
  Using the anti-anti-concentration property of the univariate Gaussian distribution, i.e., that
  $ \pr_{r \sim \normal(0,1)}[ 0 \leq r \leq t] \geq \min(t/2, 2/3)$ and the fact that
  $\alpha \leq 1/4$, we obtain that $|t_1|/\snorm{2}{\vec w \vec \Sigma^{1/2}} \leq \alpha$.
  From Equation~\eqref{eq:correlation-lower-bound}, we obtain that
  \[
    \snorm{2}{\vec w \vec \Sigma^{1/2}} = \sqrt{1 - (1- \sigma^2) (\vec w \cdot \vec w_0)^2}
    \leq \sqrt{ \sigma^2 \frac{1 + \tan^2(\pi a)}{1 + \sigma^2 \tan^2(\pi \alpha)}}
    \leq 2 \sigma\,,
  \]
  using the fact that $\sigma < 1$ and $\alpha < 1/4$.
  Therefore, we conclude that $|t_1| \leq 2 \sigma \alpha = O(\sigma \alpha)$.

This concludes the proof of Proposition~\ref{prop:error-acc}.

\bibliographystyle{alpha}
\bibliography{allrefs}

\newcommand{\etalchar}[1]{$^{#1}$}
\begin{thebibliography}{DGK{\etalchar{+}}20}

\bibitem[ABL17]{ABL17}
P.~Awasthi, M.~F. Balcan, and P.~M. Long.
\newblock The power of localization for efficiently learning linear separators
  with noise.
\newblock {\em J. {ACM}}, 63(6):50:1--50:27, 2017.

\bibitem[Dan15]{Daniely15}
A.~Daniely.
\newblock A {PTAS} for agnostically learning halfspaces.
\newblock In {\em Proceedings of The 28th Conference on Learning Theory, {COLT}
  2015}, pages 484--502, 2015.

\bibitem[Dan16]{Daniely16}
A.~Daniely.
\newblock Complexity theoretic limitations on learning halfspaces.
\newblock In {\em Proceedings of the 48th Annual Symposium on Theory of
  Computing, {STOC} 2016}, pages 105--117, 2016.

\bibitem[DGK{\etalchar{+}}20]{DGKKS20}
I.~Diakonikolas, S.~Goel, S.~Karmalkar, A.~R. Klivans, and M.~Soltanolkotabi.
\newblock Approximation schemes for relu regression.
\newblock In Jacob~D. Abernethy and Shivani Agarwal, editors, {\em Conference
  on Learning Theory, {COLT} 2020}, volume 125 of {\em Proceedings of Machine
  Learning Research}, pages 1452--1485. {PMLR}, 2020.

\bibitem[DHK{\etalchar{+}}10]{DHK+:10}
I.~Diakonikolas, P.~Harsha, A.~Klivans, R.~Meka, P.~Raghavendra, R.~A.
  Servedio, and L.~Y. Tan.
\newblock Bounding the average sensitivity and noise sensitivity of polynomial
  threshold functions.
\newblock In {\em STOC}, pages 533--542, 2010.

\bibitem[DKKZ20]{diakonikolas2020algorithms}
I.~Diakonikolas, D.~M. Kane, V.~Kontonis, and N.~Zarifis.
\newblock Algorithms and sq lower bounds for pac learning one-hidden-layer relu
  networks.
\newblock In {\em Conference on Learning Theory}, pages 1514--1539. PMLR, 2020.

\bibitem[DKN10]{DKN10}
I.~Diakonikolas, D.~M. Kane, and J.~Nelson.
\newblock Bounded independence fools degree-$2$ threshold functions.
\newblock In {\em FOCS}, pages 11--20, 2010.

\bibitem[DKPZ21]{DKPZ21}
I.~Diakonikolas, D.~M. Kane, T.~Pittas, and N.~Zarifis.
\newblock The optimality of polynomial regression for agnostic learning under
  gaussian marginals.
\newblock {\em coRR}, 2021.

\bibitem[DKS18]{DKS18a}
I.~Diakonikolas, D.~M. Kane, and A.~Stewart.
\newblock Learning geometric concepts with nasty noise.
\newblock In {\em Proceedings of the 50th Annual {ACM} {SIGACT} Symposium on
  Theory of Computing, {STOC} 2018}, pages 1061--1073, 2018.

\bibitem[DKTZ20]{DKTZ20c}
I.~Diakonikolas, V.~Kontonis, C.~Tzamos, and N.~Zarifis.
\newblock Non-convex {SGD} learns halfspaces with adversarial label noise.
\newblock In Hugo Larochelle, Marc'Aurelio Ranzato, Raia Hadsell,
  Maria{-}Florina Balcan, and Hsuan{-}Tien Lin, editors, {\em Advances in
  Neural Information Processing Systems 33: Annual Conference on Neural
  Information Processing Systems 2020, NeurIPS 2020}, 2020.

\bibitem[DKZ20]{DKZ20}
I.~Diakonikolas, D.~M. Kane, and N.~Zarifis.
\newblock Near-optimal {SQ} lower bounds for agnostically learning halfspaces
  and relus under gaussian marginals.
\newblock {\em CoRR}, abs/2006.16200, 2020.

\bibitem[DRST14]{DRST14}
I.~Diakonikolas, P.~Raghavendra, R.~A. Servedio, and L.~Y. Tan.
\newblock Average sensitivity and noise sensitivity of polynomial threshold
  functions.
\newblock {\em {SIAM} J. Comput.}, 43(1):231--253, 2014.

\bibitem[FGKP06]{FGK+:06short}
V.~Feldman, P.~Gopalan, S.~Khot, and A.~Ponnuswami.
\newblock New results for learning noisy parities and halfspaces.
\newblock In {\em Proc. FOCS}, pages 563--576, 2006.

\bibitem[FS97]{FreundSchapire:97}
Y.~Freund and R.~Schapire.
\newblock A decision-theoretic generalization of on-line learning and an
  application to boosting.
\newblock {\em Journal of Computer and System Sciences}, 55(1):119--139, 1997.

\bibitem[GGK20]{GGK20}
S.~Goel, A.~Gollakota, and A.~R. Klivans.
\newblock Statistical-query lower bounds via functional gradients.
\newblock {\em CoRR}, abs/2006.15812, 2020.

\bibitem[GR06]{GR:06}
V.~Guruswami and P.~Raghavendra.
\newblock {Hardness of learning halfspaces with noise}.
\newblock In {\em Proc.\ 47th IEEE Symposium on Foundations of Computer Science
  (FOCS)}, pages 543--552. IEEE Computer Society, 2006.

\bibitem[Hau92]{Haussler:92}
D.~Haussler.
\newblock {Decision theoretic generalizations of the PAC model for neural net
  and other learning applications}.
\newblock {\em Information and Computation}, 100:78--150, 1992.

\bibitem[HKM14]{HKM14}
P.~Harsha, A.~R. Klivans, and R.~Meka.
\newblock Bounding the sensitivity of polynomial threshold functions.
\newblock {\em Theory of Computing}, 10:1--26, 2014.

\bibitem[Kan11]{Kane11}
D.~M. Kane.
\newblock The gaussian surface area and noise sensitivity of degree-\emph{d}
  polynomial threshold functions.
\newblock {\em Computational Complexity}, 20(2):389--412, 2011.

\bibitem[Kan14]{Kane14}
D.~M. Kane.
\newblock The average sensitivity of an intersection of half spaces.
\newblock In {\em Symposium on Theory of Computing, {STOC} 2014}, pages
  437--440, 2014.

\bibitem[KKMS08]{KKMS:08}
A.~Kalai, A.~Klivans, Y.~Mansour, and R.~Servedio.
\newblock Agnostically learning halfspaces.
\newblock {\em SIAM Journal on Computing}, 37(6):1777--1805, 2008.

\bibitem[KLS09]{KLS09}
A.~Klivans, P.~Long, and R.~Servedio.
\newblock Learning halfspaces with malicious noise.
\newblock To appear in \emph{Proc.\ 17th Internat. Colloq. on Algorithms,
  Languages and Programming (ICALP)}, 2009.

\bibitem[KOS08]{KOS:08}
A.~Klivans, R.~O'Donnell, and R.~Servedio.
\newblock Learning geometric concepts via {G}aussian surface area.
\newblock In {\em Proc.\ 49th IEEE Symposium on Foundations of Computer Science
  (FOCS)}, pages 541--550, 2008.

\bibitem[KSS94]{KSS:94}
M.~Kearns, R.~Schapire, and L.~Sellie.
\newblock {Toward Efficient Agnostic Learning}.
\newblock {\em Machine Learning}, 17(2/3):115--141, 1994.

\bibitem[KTZ19]{KTZ19}
V.~Kontonis, C.~Tzamos, and M.~Zampetakis.
\newblock Efficient truncated statistics with unknown truncation.
\newblock In {\em 2019 IEEE 60th Annual Symposium on Foundations of Computer
  Science (FOCS)}, pages 1578--1595. IEEE, 2019.

\bibitem[MP68]{MinskyPapert:68}
M.~Minsky and S.~Papert.
\newblock {\em {P}erceptrons: an introduction to computational geometry}.
\newblock MIT Press, Cambridge, MA, 1968.

\bibitem[MT94]{MT:94}
W.~Maass and G.~Turan.
\newblock How fast can a threshold gate learn?
\newblock In S.~Hanson, G.~Drastal, and R.~Rivest, editors, {\em Computational
  Learning Theory and Natural Learning Systems}, pages 381--414. MIT Press,
  1994.

\bibitem[Nov62]{Novikoff:62}
A.~Novikoff.
\newblock On convergence proofs on perceptrons.
\newblock In {\em Proceedings of the Symposium on Mathematical Theory of
  Automata}, volume XII, pages 615--622, 1962.

\bibitem[O'D14]{AoBF14}
R.~O'Donnell.
\newblock {\em Analysis of Boolean Functions}.
\newblock Cambridge University Press, 2014.

\bibitem[Ros58]{Rosenblatt:58}
F.~Rosenblatt.
\newblock The {P}erceptron: a probabilistic model for information storage and
  organization in the brain.
\newblock {\em Psychological Review}, 65:386--407, 1958.

\bibitem[SSBD14]{SB14}
S.~Shalev-Shwartz and S.~Ben-David.
\newblock {\em Understanding machine learning: From theory to algorithms}.
\newblock Cambridge university press, 2014.

\bibitem[STC00]{CristianiniShaweTaylor:00}
J.~Shawe-Taylor and N.~Cristianini.
\newblock {\em An introduction to support vector machines}.
\newblock Cambridge University Press, 2000.

\bibitem[Sze67]{Sze67}
G.~Szeg{\"o}.
\newblock {\em Orthogonal Polynomials}.
\newblock Number $\tau$. 23 in American Mathematical Society colloquium
  publications. American Mathematical Society, 1967.

\bibitem[Val84]{val84}
L.~G. Valiant.
\newblock A theory of the learnable.
\newblock In {\em Proc.\ 16th Annual ACM Symposium on Theory of Computing
  (STOC)}, pages 436--445. ACM Press, 1984.

\bibitem[Vap98]{Vapnik:98}
V.~Vapnik.
\newblock {\em Statistical Learning Theory}.
\newblock Wiley-Interscience, New York, 1998.

\bibitem[Ver18]{Ver18}
R.~Vershynin.
\newblock {\em High-Dimensional Probability: An Introduction with Applications
  in Data Science}.
\newblock Cambridge Series in Statistical and Probabilistic Mathematics.
  Cambridge University Press, 2018.

\end{thebibliography}
\clearpage
\appendix

\section{Hermite Polynomials}
We are also going to use the Hermite polynomials that form an orthonormal system
with respect to the Gaussian measure.  We denote by $L^2(\R^d, \normal(\vec 0, \vec I))$ the
vector space of all functions $f:\R^d \to \R$ such that $\E_{\vec x \sim
                \normal(\vec 0, \vec I)}[f^2(\x)] < \infty$.  The standard inner product for this space is
$f\cdot g:=\E_{\vec x \sim \normal(\vec 0, \vec I)}[f(\vec x) g(\vec x)]$.
While usually one considers the probabilists' or physicists' Hermite polynomials,
in this work we define the \emph{normalized} Hermite polynomial of degree $i$ to be
\(
H_0(x) = 1, H_1(x) = x, H_2(x) = \frac{x^2 - 1}{\sqrt{2}},\ldots,
H_i(x) = \frac{He_i(x)}{\sqrt{i!}}, \ldots
\)
where by $He_i(x)$ we denote the probabilists' Hermite polynomial of degree
$i$.  These normalized Hermite polynomials form a complete orthonormal basis
for the single-dimensional version of the inner product space defined above. To
get an orthonormal basis for $L^2(\R^d, \normal(\vec 0, \vec I))$, we use a multi-index
$\alpha \in \N^d$ to define the $d$-variate normalized Hermite polynomial as $H_\alpha(\vec x) =
        \prod_{i=1}^d H_{\alpha_i}(\x_i)$.
The total degree of $H_\alpha$ is $|\alpha| := \sum_{i} \alpha_i$.
Given a function $f \in L^2(\R^d, \normal(\vec 0, \vec I))$, we compute its Hermite coefficients as
\(
\hat{f}(\alpha) = \E_{\vec x\sim \normal(\vec 0, \vec I)} [f(\vec x) H_\alpha(\vec x)]
\)
and express it uniquely as
\(
\sum_{\alpha \in \N^d} \hat{f}(\alpha) H_\alpha(\vec x).
\)
For more details on the Gaussian space and Hermite analysis, we refer the reader to
\cite{AoBF14}.  Most of the facts about Hermite polynomials that we use in this
work are well-known properties and can be found, for example, in \cite{Sze67}.
We are going to use the following simple fact about the gradient of Hermite polynomials;
for a proof see, for example, Lemma 6 of \cite{KTZ19}.
\begin{fact}
        \label{fct:hermite-gradient}
        Let $f \in L^2(\R^d, \normal(\vec 0, \vec I))$. It holds
        $\E_{\vec x \sim \normal(\vec 0,\vec I)}
                [(\nabla H_\alpha(\vec x) \cdot \vec e_i)^2] = \sum_{\alpha \in \N^d} \alpha_i \hat{f}(\alpha) \,.$
\end{fact}

\section{Omitted Proofs from Section~\ref{sec:structural}}\label{app:structural}

\subsection{Proof of Lemma~\ref{lem:polynomial-regr}}\label{ssec:polynomial-regr}
We restate and prove the following lemma:
\begin{lemma}[$\ell_2$-Polynomial Regression]
Let $\D$ be a distribution on $\R^d\times\{\pm 1\}$
whose $\x$-marginal is $\normal{(\vec 0, \vec I)}$. Let $k\in \Z_+$ and $\eps, \delta>0$.
There is an algorithm that draws $N=(d k)^{O(k)}\log(1/\delta)/\eps^2$
samples from $\D$, runs in time $\poly(N,d)$, and outputs a polynomial $P(\x)$ of
degree at most $k$ such that
$\E_{\x\sim \D}[(f(\x)-P(\x))^2]\leq \min_{P'\in {\cal P}_k} \E_{\x\sim \D}[(f(\x)-P'(\x))^2]+\eps$,
with probability $1-\delta$.
\end{lemma}
\begin{proof}
Let $S$ denote the empirical distribution of $\D$ with $N=(d/\eps)^{O(k)}$ samples.
Recall that for any such $P(\x)$, it holds that $\E_{\x\sim \D}[P^2(\x)]\leq 5$ (see Lemma~\ref{lem:dimension-bound}). Writing $P(\x)$ in the Hermite basis, $P(\x)=\sum_{\alpha\in \N^d} c_{\alpha}H_\alpha(\x)$, it holds that
$\sum_{\alpha\in \N^d} c_{\alpha}^2 =\E_{\x\sim \D}[P^2(\x)]$. The one-dimensional Hermite polynomials of $k$-degree are $H_k(z)=\sum_{m=0}^{\lfloor k/2\rfloor}\frac{(-1)^m z^{k-2m}}{m!(n-2m)! 2^{m}}$. Thus, each monomial has coefficient absolute bounded by $2^k$. Therefore, the maximum coefficient of a multidimensional Hermite polynomial $H_a(\x)$ is $2^{|a|}$, thus the maximum coefficient of the polynomial $P(\x)$ is $O(2^k)$.
Let us now prove that for any degree-$k$ polynomial $P(\x)$ with coefficients bounded by $C=2^{O(k)}$, we have
\[\left|\E_{(\x,y)\sim S}[P(\x)y]- \E_{(\x,y)\sim \D}[P(\x)y]\right|  \leq \eps \;,\]
with high constant probability. Write $P(\x)=\sum a_i m_i(\x)$, where the summation ranges over all monomials $m_i$  with degree less than $k$ along with their coefficients $a_i$. We have
\begin{equation}\label{eq:bound-polynomial}
\left|\E_{(\x,y)\sim S}[P(\x)y]- \E_{(\x,y)\sim \D}[P(\x)y]\right| \leq \sum |a_i|\, \left|\E_{(\x,y)\sim S}[m_i(\x)y]- \E_{(\x,y)\sim \D}[m_i(\x)y]\right|\;.
\end{equation}
Using Markov's inequality, we have
\begin{align*}
\pr\left[\left|\E_{(\x,y)\sim S}[m_i(\x)y]- \E_{(\x,y)\sim \D}[m_i(\x)y]\right|\geq \eps/(d^k C)\right] & \leq \frac{C^2 d^{2k}}{N\eps^2}\var[m_i(\x)y]
\\                                                                                                                                             & \leq \frac{C^2d^{2k}}{N\eps^2}\E_{(\x,y)\sim \D}[m_i^2(\x)y^2]
\\& \leq \frac{C^2d^{2k}}{N\eps^2}\E_{(\x,y)\sim \D}[\snorm{2}{\x}^{2i}]=O\left(\frac{C^2i^{i}d^{2k}}{N\eps^2}\right)\;.
\end{align*}
By using the fact that $N=(d\ k)^{O(k)}/\eps^2 $ and applying above to the Equation~\eqref{eq:bound-polynomial}, we have
\[
  \left|\E_{(\x,y)\sim S}[P(\x)y]- \E_{(\x,y)\sim \D}[P(\x)y]\right| \leq C\sum  \left|\E_{(\x,y)\sim S}[m_i(\x)y]- \E_{(\x,y)\sim \D}[m_i(\x)y]\right|\leq \eps\;,
\]
with high probability.
Next, we need to bound the difference $\left|\E_{(\x,y)\sim S}[P^2(\x)]- \E_{(\x,y)\sim \D}[P^2(\x)]\right|$. This can be done by applying the same procedure as before and noting that the highest coefficient is at most $C^2$ and the degree is $2k$.
Thus, for any $k$-degree polynomial $P$ with high probability, we have
\begin{equation}\label{eq:any-polynomial}
\left| \E_{(\x,y)\sim S}[(P(\x)-y)^2]- \E_{(\x,y)\sim \D}[(P(\x)-y)^2]\right| \leq \eps\;,
\end{equation}
where we used the fact that $\E_{(\x,y)\sim S}[y^2] = \E_{(\x,y)\sim \D}[y^2]$.
By solving a convex program, we can find a polynomial $P$ such that
\[
  \E_{(\x,y)\sim S}[(y-P(\x))^2]\leq \min_{{P'\in {\cal P}_k}} \E_{(\x,y)\sim S}[(y-P'(\x))^2]+\eps\;,
\]
Note that if $P''(\x)=\argmin_{{P'\in {\cal P}_k}} \E_{(\x,y)\sim \D}[(y-P'(\x))^2]$, then
\[
  \min_{{P'\in {\cal P}_k}} \E_{(\x,y)\sim S}[(y-P'(\x))^2]\leq \E_{(\x,y)\sim S}[(y-P''(\x))^2]\leq \E_{(\x,y)\sim \D}[(y-P''(\x))^2]\;,
\]
where we used Equation~\eqref{eq:any-polynomial}. Thus, we have proved that
\[
  \E_{(\x,y)\sim \D}[(y-P(\x))^2]\leq \min_{{P'\in {\cal P}_k}} \E_{(\x,y)\sim \D}[(y-P'(\x))^2]+\eps\;,
\]
with high constant probability. Using basic boosting procedures (see, e.g., exercise 1, chapter 13 of \cite{SB14}), we can boost the probability of success to $1-\delta$, with $N'=N\log(1/\delta)=(d k)^{O(k)}\log(1/\delta)/\eps^2$.
\end{proof}
\section{Omitted Proofs from Section~\ref{sec:ptas}}\label{app:claim1}
We restate and prove the following claims.

 \begin{claim}
    For vectors $\vec v, \vec u \in \R^d, t \in \R$ define the halfspaces $h_0(\x) = \sgn(\vec u \cdot \x)$,
    $h_1(\x) = \sgn(\vec v \cdot \x+ t)$.
    It holds
    $\pr_{\x \sim \normal(\vec 0,\vec I)}[h_1(\x) \neq h_0(\x)]
      \geq \theta( \vec u, \vec v)/\pi$.
  \end{claim}
   \begin{proof}
     Denote $\theta = \theta(\vec u, \vec v)$ and first assume that $\theta \in
       [0, \pi/2)$.  Since the Gaussian distribution is invariant under rotations,
     for simplicity we may assume that $\vec u, \vec v$ span $\R^2$.
     Morover, assume that two halfspaces intersect at the origin $(0,0)$ (if they do not intersect then
     the claimed lower bound on the disagreement
     $\pr_{\x \sim \normal(\vec 0,\vec I)}[h_1(\x) \neq h_0(\x)]$
     is trivially true as their angle is $0$).
     Moreover, assume that $\vec u = \vec e_1$ and that the Gaussian is centered at
     some point $(z,0)$, i.e., a point that lies on the $\x_1$-axis.
     This follows from the fact that $h_0(\x) = \sgn(\vec u \cdot \x)$ is homogeneous.
     After we change coordinates, the halfspace $h_1$ is also
     homogeneous, with normal vector $\vec v = (-\sin \theta, \cos\theta)$.
     We will show that the disagreement between the two halfspaces is \emph{minimized} where $z=0$, i.e., when
     the Gaussian is centered on the intersection point of the two halfspaces.
     Using the above, we obtain that
     \begin{align*}
       \pr_{\x \sim \normal(\vec 0,\vec I)} & [h_1(\x) \neq h_0(\x)]
       \\     & = \int_{-\infty}^0 \int_{\x_1 \tan \theta }^0 e^{-((\vec x_1 - z)^2/2 - \vec x_2^2/2} \d \x_2 \d \x_1
             +
             \int_{0}^\infty \int_{0}^{\x_1 \tan \theta}  e^{-((\vec x_1 - z)^2/2 - \vec x_2^2/2} \d \x_2 \d \x_1
             := q(z)
     \end{align*}
     We will show that the function $q$ is minimized for $z = 0$.
     Taking the derivative with respect to $z$, we obtain
     \[
       q'(z) =
       \int_{-\infty}^0 \int_{\x_1 \tan \theta }^0 (\vec x_1 - z) e^{-((\vec x_1 - z)^2/2 - \vec x_2^2/2} \d \x_2 \d \x_1
       +
       \int_{0}^\infty \int_{0}^{\x_1 \tan \theta}  (\vec x_1 - z) e^{-((\vec x_1 - z)^2/2 - \vec x_2^2/2} \d \x_2 \d \x_1 \,.
     \]
     Observe that $q'(-z) = - q'(z)$, i.e, $q'(z)$ is an odd function with $q'(0) = 0$.  Thus, it can only be minimized
     at $0$.  We have that $q''(0) > 0$ and therefore $z = 0$ is the global minimizer of $q(z)$.
     The case $\theta \in [\pi/2,\pi]$ can be shown similarly.
   \end{proof}

\begin{claim}
    Let $h_1(\x) = \sgn(\vec u \cdot \x + t_1)$, $h_2(\x) = \sgn(\vec v \cdot \x + t_2)$
    be two halfspaces.
    Let $r_1 = t_1/\snorm{2}{\vec u \vec \Sigma^{1/2}}$, $r_2 = t_2/\snorm{2}{\vec v \vec \Sigma^{1/2}}$.
It holds
    \[
      \pr_{\x \sim \normal(\vec 0, \vec \Sigma)}[ h_1(\x) \neq h_2(\x)]
      \geq
      \pr_{r \sim \normal(0,1)}[\min(r_1,r_2) \leq r \leq \max(r_1,r_2)]
      \,.
\]
  \end{claim}
  \begin{proof}
     We first observe that
     \begin{align*}
       \pr_{\x \sim \normal(\vec 0, \vec \Sigma)}[ h_1(\x) \neq h_2(\x)]
        & =
       \pr_{\x \sim \normal(\vec 0, \vec I)}[ h_1(\vec \Sigma \x) \neq h_2(\vec \Sigma\x)]
       \\
        & \geq \left| \pr_{\x \sim \normal(\vec 0, \vec I)}[ h_1(\vec \Sigma\x) \neq 0]-\pr_{\x \sim \normal(\vec 0, \vec I)}[ h_2(\vec \Sigma\x) \neq 0] \right|\;,
     \end{align*}
     where in the last step we used triangle inequality. Moreover, using that $\pr_{\x \sim \normal(\vec 0, \vec I)}[ h_1(\vec \Sigma\x) \neq 0]=\E_{\x \sim \normal(\vec 0, \vec I)}[h_1(\vec \Sigma\x)]$, we have
     \[
       \pr_{\x \sim \normal(\vec 0, \vec \Sigma)}[ h_1(\x) \neq h_2(\x)]\geq  \left|\E_{\x \sim \normal(\vec 0, \vec I)}[ h_1(\vec \Sigma\x)]-\E_{\x \sim \normal(\vec 0, \vec I)}[ h_2(\vec \Sigma\x)] \right|
     \;.\]
     Note that $h_1(\vec \Sigma\x)=\sgn(\vec u \cdot \vec \Sigma\x/\snorm{2}{\vec u \vec \Sigma^{1/2}} + r_1)$ and $h_2(\vec \Sigma\x)=\sgn(\vec v \cdot \vec \Sigma\x/\snorm{2}{\vec v \vec \Sigma^{1/2}} + r_2)$, thus
     \begin{align*}
       \pr_{\x \sim \normal(\vec 0, \vec \Sigma)}[ h_1(\x) \neq h_2(\x)]
        & \geq\left|\E_{\x \sim \normal(\vec 0, \vec I)}[ h_1(\vec \Sigma\x)]-\E_{\x \sim \normal(\vec 0, \vec I)}[ h_2(\vec \Sigma\x)] \right|
       \\
        & = \left|\pr_{r \sim \normal(0,1)}[ r \leq r_1]-\pr_{r \sim \normal(0,1)}[ r \leq r_2] \right|\;,
     \end{align*}
     which completes the proof.
   \end{proof}

\section{Agnostic Proper Learning of ReLus} \label{sec:relus-ptas}
In this section, we use our techniques to develop a proper agnostic learning algorithm
that handles ReLU activations.  We work in the standard $L_2$-regression
setting, i.e., we want to find a weight vector $\vec w$ such that
\[
  \E_{(\x,y)\sim \D}[(y-\max(0, \vec w \cdot \x))^2]\leq \min_{f \in {\cal C}^\rho}\E_{(\x,y)\sim \D}[(y-f(\x))^2] + \eps \;,
\]
where by ${\cal C}^{\rho}$ we denote the class of ReLU activations,
i.e., ${\cal C}^{\rho} = \{\vec x \mapsto \max(0,  \vec w \cdot \vec x + t): \|\vec w\|_2 \leq 1 ,\vec w\in\R^d, t\in \R^d \}$.
Moreover, we are going to use ${\cal C}_0^{\rho} = \{\vec x \mapsto \max(0,  \vec w \cdot \vec x): \|\vec w\|_2\leq 1 ,\vec w\in\R^d  \}$ and $\genf(\x)$ to denote the ReLU activation function. Finally, observe that for $\vec w\in \R^d$ with $\|\vec w\|_2 \leq 1$ it holds $\max(0,  \vec w \cdot \vec x)=\snorm{2}{\vec w}\max(0,  \vec w \cdot \vec x/\snorm{2}{\vec w})$.
To keep the presentation simple we are going to assume, similarly to \cite{DGKKS20}
  that the observed labels $y$ are bounded in $[-1,1]$.
For the rest of the section, we assume that for the labels $y$, it holds $|y|<1$.

The pseudocode of our algorithm is given in Algorithm~\ref{alg:proper-learner-general}.

\begin{algorithm}[H]
\caption{Agnostic Proper Algorithm for ReLU Regression} \label{alg:proper-learner-general}
\begin{algorithmic}[1]
\Procedure{Agnostic-learner}{$\eps, \delta, \D$}\\
\textbf{Input:} $\eps>0$, $\delta>0$ and  distribution $\D$\\
\textbf{Output:} A hypothesis $h\in{\cal C}$ such as $\E_{(\x,y)\sim \D}[(h(\x)-y)^2]\leq \min_{f\in {\cal C}^{\genf}}\E_{(\x,y)\sim \D}[(f(\x)-y)^2]+\eps$ with probability $1-\delta$.\\
\State $k\gets  C/\eps^{4/3}$, $\eta\gets \eps^2/C$.
\State Find $P(\x)$ such $\E_{(\x,y)\sim \D}[(y-P(\x))^2]\leq \min_{P'\in{\cal P}_{k}}\E_{(\x,y)\sim \D}[(y-P'(\x))^2]+O(\eps^3)$.
\State Let $\vec M=\E_{\x\sim \D_\x}[\nabla P(\x)\nabla P(\x)^\top]$.
\State Let $V$ be the subspace spanned by the eigenvectors of $\vec M$ whose eigenvalues are at least $\eta$.
\State Construct an $\eps$-cover ${\cal H}$ of hypotheses with normal vectors in $V$ \Comment{}{see Fact~\ref{fct:cover}}.
\State Draw $\Theta(\frac{1}{\eps^2}\log(|{\cal H}|/\delta))$ i.i.d.\ samples
from $\D$ and construct the empirical distribution $\widehat\D$.
\State $h\gets \argmin_{h'\in {\cal H}} \E_{(\x,y)\sim \widehat{\D}}[(h'(\x)-y)^2]$\label{alg:emprical-outputs-gen}
\State $\textbf{return } h$.
\EndProcedure
\end{algorithmic}
\end{algorithm}

\begin{theorem}\label{thm:proper-learner-gen}
Let $\D$ be a distribution on $\R^d\times\R$ whose $\x$-marginal is $\normal(\vec 0, \vec I)$. Algorithm~\ref{alg:proper-learner-general} draws $N=d^{O(1/\eps^{4/3})}\poly(1/\eps)\log(1/\delta)$ samples from $\D$, runs in time $  ( d^{O(1/\eps^{4/3})} + (1/\eps)^{O(1/\eps^{10/3})} ) \log(1/\delta)$, and computes a hypothesis
$h\in{\cal C}^\genf$ such that, with probability at least $1-\delta$, we have that 
$$\E_{(\x,y)\sim \D}[(y-h(\x))^2]\leq \min_{f \in {\cal C}_0^\rho}\E_{(\x,y)\sim \D}[(y-f(\x))^2] + \eps \;.$$
\end{theorem}

The main structural result of this section is the following proposition showing that we can
perform dimension reduction by looking at high-influence directions of a low-degree polynomial.

\begin{proposition}\label{prop:structural-gen}
Let $C$ be a sufficiently large constant, fix $\eps>0$, $k = C/\eps^{4/3}$.
Let $P(\x) \in \mathcal{P}_k$ be a polynomial such that
\[
    \E_{(\x,y)\sim \D}[(y-P(\x))^2]\leq \min_{P' \in \mathcal{P}_k}\E_{(\x,y)\sim \D}[(y-P'(\x))^2]+O(\eps^3)
    \,.
\]
Moreover, let $\vec M= \E_{\x\sim \D_\x}[\nabla P(\x)\nabla P(\x)^\top]$ and
$V$ be the subspace spanned by the eigenvectors of $\vec M$ with eigenvalues
larger than $\eta$, where $\eta=\eps^2/C$. Then, for any function
$f\in {\cal C}_0^{\genf}$, it holds
\[
  \min_{\vec v\in V, \|\vec v\|_2 \leq 1, t\in \R} \E_{(\x,y)\sim \D}[(\rho(\vec v \cdot \x + t) -y)^2]
  \leq
  \E_{(\x,y)\sim \D}[(f(\x) - y)^2] + \eps \,.
\]
\end{proposition}
The proof of Proposition~\ref{prop:structural-gen} is similar to the proof of Proposition~\ref{prop:structural}.
We provide the details below for completeness.
\begin{proof}
Suppose for the sake of contradiction that there exists a hypothesis $f\in {\cal C}_0^{\genf}$ such that for every hypothesis
$f'\in {\cal C}^{\genf}_{V}$, it holds
\begin{equation}\label{eq:contradiction-gen}
  \min_{\vec v\in V, \|\vec v\|_2 \leq 1, t\in \R}
  \E_{(\x,y)\sim \D}[(\rho(\vec v \cdot \x + t) -y)^2]
  >
  \E_{(\x,y)\sim \D}[(f(\x) - y)^2] + \eps \,.
\end{equation}
Equivalently, from the above equation, we have that for every $\vec v \in V$
with $\|\vec v\|_2 \leq 1$ and $t \in \R$:
\begin{equation}
  \label{eq:correlation_bound}
  2 \E_{(\x,y)\sim \D}[(f(\x) - \rho(\vec v \cdot \x + t)) y ]
  >
  \eps + \E_{\x\sim \D_\x}[f^2(\x)] - \E_{\x \sim \D_\x}[\rho^2(\vec v \cdot \x  + t )]
  \,.
\end{equation}
\begin{claim}\label{clm:contradiction-gen}
It suffices to show that there exists some polynomial $Q(\x)$ of degree at most $k$, 
with $\E_{\x\sim \D_\x}[Q^2(\x)]\leq 4$,
that $(\eps/4)$-correlates  with $(y- P(\x))$, i.e.,
\[\E_{(\x,y)\sim \D}[Q(\x) (y-P(\x))]\geq \eps/4\;.\]
\end{claim}
\begin{proof}
We have that the polynomial $P(\x)+ \zeta Q(\x)$,
for $\zeta= \Theta( \eps) $, has degree at most $k$ and
decreases the value of $\E_{(\x,y)\sim \D}[(y-P(\x))^2]$
by at least $\Omega(\eps^2)$, which contradicts the optimality of
$P(\x)$, i.e., that $P(\x)$, $O(\eps^3)$-close to the polynomial that minimizes the
$L_2$ error with $y$.
\end{proof}

We now construct such a polynomial $Q(\x)$.
We have $f(\x)=
    \genf(\vec w \cdot \x )
    =
    \genf(\vec w_{V} \cdot \x + \vec w_{V^\perp} \cdot \x)$, for some $0<a\leq1$.
It holds that $\vec w_{V^\perp} \neq \vec 0$ since otherwise we
would have that $f \in {\cal C}^{\genf}_V$.
For simplicity, we denote $\vec \xi = \vec w_{V^\perp}/\snorm{2}{ \vec w_{V^\perp}}$.  Notice that
the direction $\vec \xi$ is of low influence since $\vec \xi \in V^\perp$.
Recall, that by $\D_{\vec \xi}$ we denote the projection of $\D$ onto the (one-dimensional) subspace
spanned by $\vec \xi$.
We define $f_{V}(\x)=\E_{\vec z \sim \D_{\vec \xi }}[f(\vec z + \x_{V})]$:
a convex combination of hypotheses in ${\cal C}^{\genf}_V$.
In particular, $f_V(\x)$ is a smoothed version of the hypothesis $\genf(\vec w_V \cdot \x +t)$, whose normal vector belongs in $V$.
We first observe that by \eqref{eq:correlation_bound} $f_V(\x)$ cannot correlate too well with $y$:
\begin{align}
  \label{eq:smooth-correlation-lb}
  2 \E_{(\x,y)\sim \D} [(f(\x) - f_V(\x)) y ]
  &=
  2 \E_{\vec z \sim \D_{\vec \xi}}[
  \E_{(\vec x, y) \sim \D} [(f(\x) - \rho(\vec w \cdot \x_V + \vec w \cdot \vec z)) y ]  ]
  \nonumber
  \\
  &\geq \eps + \E_{\vec x \sim \D_\x}[f^2(\x)]
  -
  \E_{\vec z \sim \D_{\vec \xi}} [ \E_{\vec x \sim \D_{\vec x}} [\rho^2(\vec w \cdot \x_V + \vec w \cdot \vec z) ] ]
  = \eps\,,
\end{align}
where the last equality follows by the fact that
\[
  \E_{\vec x \sim \D_\x}[f^2(\x)]
  = \E_{\vec u \sim \D_{\vec \xi^\perp}}\big[ \E_{\vec z \sim \D_\vec \xi}[f^2(\vec u + \vec z)] \big]
  = \E_{\vec z \sim \D_\vec \xi} \big[\E_{\vec u \sim \D_{\vec \xi^\perp}}
    [\rho^2(\vec u \cdot \vec w_{V} + \vec w \cdot \vec z)] \big] \
  = \E_{\vec z \sim \D_\vec \xi} \big[\E_{\vec u \sim \D_{\vec x}}
    [\rho^2( \vec u_V \cdot \vec w  + \vec w \cdot \vec z)] \big]\,.
\]
Our argument consists two main claims.  We first show that the function $f(\x) - f_V(\x)$ correlates
non-trivially with $y - P(\x)$.  Then we show that we can approximate $f(\x)-f_V(\x)$ by
a low degree polynomial $Q(\x)$ that maintains non-trivial correlation with $y - P(\x)$,
see Claim~\ref{clm:polynomial-apx-gen}.  We start by proving the first claim.
\begin{claim}\label{clm:correlation-with-new-f-gen}
It holds
\[
    \E_{\x\sim \D_\x}[(f(\x)-f_V(\x))(y-P(\x))] \geq \eps/2 -  \sqrt{2\eta} \,.
\]
\end{claim}
\begin{proof}
We have $f_{V}(\x)=\E_{\vec z \sim \D_{\vec \xi }}[f(\vec z + \x_{\vec \xi^\perp})]
    = \E_{\vec z \sim \D_{\vec \xi }}[\genf( \vec w_V \cdot \x_{V} + \vec w \cdot \vec z )]
$ and, since $f_{V}$ is a convex combination of hypothesis in ${\cal C}^{\genf}_V$,
from Equation~\eqref{eq:contradiction-gen}, we see that $\E_{(\x,y)\sim \D}[(f(\x)-f_V(\x))y]\geq \eps$.
Thus, we have
\begin{align}
\label{eq:correlation-gen}
\E_{(\x,y)\sim \D}[(f(\x) - f_V(\x)) (y-P(\x))]
 & = \E_{(\x,y)\sim \D}[(f(\x) - f_V(\x)) y]- \E_{\x\sim \D_\x}[(f(\x)-f_V(\x))P(\x)]
\nonumber
\\
 & \geq \eps/2 - \E_{\x\sim \D_\x}[(f(\x)-f_V(\x))P(\x)]
\,.
\end{align}
To deal with $\E_{\x\sim \D_\x}[(f(\x)-f_V(\x))P(\x)]$,
we first observe that for any function $g(\x)$ depending only on
the projection of $\x$ onto the subspace $\vec \xi^\perp$, i.e.,
it holds
$g(\vec x) = g(\x_{\vec\xi^\perp})$,
we have
\[
    \E_{\x \sim \D_\x}[(f(\x)-f_V(\x))g(\x)] =
    \E_{\vec v \sim \D_{\vec \xi^{\perp}}} \left[\E_{\vec z \sim \D_{\vec \xi}} [f(\vec v + \vec z) - f_V(\vec v)] ~ g(\vec v)  \right] = 0 \,,
\]
since for every $\vec x \in \R^d$, it holds
$f_V(\vec x)
    = \E_{\vec z \sim \D_{\vec \xi}} [f(\vec x_{\vec \xi^\perp} + \vec z)]
    = \E_{\vec z \sim \D_{\vec \xi}} [f(\vec x_V + \vec z)]
$.
Unfortunately, this is not true since $P(\vec x)$ is not only
a function of $\vec x_{\vec \xi^\perp}$.
However, since $V$ contains the high influence eigenvectors
it holds that $P$ is almost a function of $\vec x_{\vec \xi^\perp}$.
In fact, we show that we can replace the polynomial $P$ by a different polynomial
of degree at most $k$ that only depends on the projection of $\x$ on $\vec \xi^\perp$.
Similarly to the definition of the ``smoothed" hypothesis $f_V$, we
define $R(\x)=\E_{\vec z \sim \D_{\vec \xi}}[P(\x_{\vec \xi^\perp} + \vec z)]$.
We first prove that $R(\x)$ is close to $P(\x)$ in the $L_2$ sense.

Now, adding and subtracting $R(\x) = \E_{\vec z \sim \D_\vec \xi}[P(\x_{\vec \xi^\perp}+ \vec z)]$, we get
\begin{align*}
\E_{\x\sim \D_\x}[(f(\x) - f_V(\x)) P(\x)] =
\E_{\x\sim \D_\x}[(f(\x)-f_V(\x))(P(\x)-R(\x_{\vec \xi^\perp}))]
+\E_{\x\sim \D_\x}[(f(\x)-f_V(\x)) R(\x_{\vec \xi^\perp})]\;.
\end{align*}
The second term equals to zero, from the fact that
$\E_{\vec z \sim \D_{\vec \xi}} [f(\vec z +\x_{\vec \xi^\perp})-f_V(\x_{\vec \xi^\perp})]=0$.
Using Cauchy-Schwarz inequality, we get
\begin{align}
\E_{\x\sim \D_\x}[(f(\x)-f_V(\x))(P(\x)-R(\x_{\vec \xi^\perp}))]
 & \leq \sqrt{\E_{\x\sim \D_\x}[(f(\x)-f_V(\x))^2]\E_{\x\sim \D_\x}[(P(\x) - R(\x_{\vec \xi^\perp}))^2]} \nonumber                 \\
 & \leq \sqrt{2}\sqrt{\E_{\x\sim \D_\x}[(P(\x)- R(\x_{\vec \xi^\perp}))^2]}\leq \sqrt{2\eta}\;,\label{eq:bound-of-polynomials-gen}
\end{align}
where we used Claim~\ref{clm:hermite-influence}. Using Equation~\eqref{eq:correlation-gen}, we get that
\[
    \E_{(\x,y)\sim \D}[(f(\x) - f_V(\x)) (y-P(\x))]\geq \eps/2-\sqrt{2\eta}\;,
\]
which completes the proof of Claim~\ref{clm:correlation-with-new-f-gen}.
\end{proof}

\begin{claim}\label{clm:polynomial-apx-gen}
There exists a  polynomial $Q(\x)$ of degree $O(1/\eps^{4/3})$ such that
$\E_{\x\sim \D_\x}[Q(\x)(y - P(\x))] \geq \eps/4 - \sqrt{2\eta}$ and  $\E_{\x\sim \D_\x}[Q^2(\x)] \leq 4$.
\end{claim}
\begin{proof}
For any polynomial $Q(\x)$, we have
\begin{align}
\E_{\x\sim \D_\x}[Q(\x)(y - P(\x))] & = \E_{\x\sim \D_\x}  [(Q(\x) + (f(\x)-f_V(\x)) - (f(\x) - f_V(\x)) )(y-P(\x)) ]\nonumber
\\   & \geq   \eps/2 - 2 \sqrt{\eta} + \E_{\x\sim \D_\x}[(Q(\x) - (f(\x) - f_V(\x)) )(y-P(\x)) ]\label{eq:connective_equ-gen}\;,
\end{align}
where we used Claim~\ref{clm:correlation-with-new-f-gen}.
By choosing $Q(\x)=S(\x)-\E_{\vec z \sim \D_{\vec \xi }}[S(\x_{\xi^\perp} +
        \xi)]$, where we denote by $S(\x)$ the Hermite expansion of $f$ truncated up to
degree $k$, $S(\x) = \sum_{|\alpha| \leq k } \hat{f}(\alpha) H_\alpha(\x)$, we show that
\[
    \E_{\x\sim \D_\x}[(f(\x)-f_V(\x)-Q(\x))^2]\leq \eps^2\;.
\]
We need the following fact:
\begin{fact}[\cite{GGK20}]\label{fct:relu-approx}
Let $f\in{\cal C}_0^{\genf}$, and let $S$ be the Hermite expansion up to $k$-degree
of $f$, i.e., $S(\x) = \sum_{|\alpha| \leq k} \hat{f}(\alpha) H_\alpha(\x)$. Then
$
    \E_{\x \sim \normal(\vec 0,\vec I)}[(S(\x)-f(\x))^2]=O(k^{-3/2})
$.
\end{fact}
Using the inequality $(a+b)^2\leq 2a^2 + 2b^2$, we get that
\[
    \E_{\x\sim \D_\x}[(f(\x)-f_V(\x)-Q(\x))^2]\leq 2\E_{\x\sim \D_\x}[(f(\x)-S(\x))^2] + 2\E_{\x\sim \D_\x}[(f_V(\x)-\E_{\vec z \sim \D_{\vec \xi }}[S(\x_{\xi^\perp})])^2]\;.
\]
Moreover, from Jensen's inequality, it holds that
\[
    \E_{\x\sim \D_\x}[(f_V(\x)-\E_{\vec z \sim \D_{\vec \xi }}[S(\x_{\xi^\perp}+\xi)])^2]\leq \E_{\x\sim \D_\x}[(f(\x)-S(\x))^2]=O\left(k^{-3/2}\right)\;,
\]
where in the last equality we used the Fact~\ref{fct:relu-approx}. Choose $k=\Theta(1/\eps^{4/3})$. Applying Cauchy-Schwartz to the Equation~\eqref{eq:connective_equ-gen}, we get
\begin{align*}
\E_{(\x,y)\sim \D}[Q(\x)(y - P(\x))] & \geq
\eps/2 - \sqrt{2\eta} - \sqrt{\E_{\x\sim \D_\x}[(Q(\x) - (f(\x) - f_V(\x)))^2]\E_{(\x,y)\sim \D}[(y-P(\x))^2]} \\&\geq\eps/4 -  \sqrt{2\eta}\;,
\end{align*}
where we used the fact that $\E_{(\x,y)\sim \D}[(y-P(\x))^2]\leq 2$.
Note that from the reverse triangle inequality it holds that
\begin{equation}\label{eq:norm-bound-gen}
\sqrt{\E_{\x\sim\D}[Q^2(\x)]}\leq \sqrt{  \E_{\x\sim \D_\x}[(f(\x)-f_V(\x))^2]} + \eps\leq \sqrt{2}+ \eps\;.
\end{equation}
Equation~\eqref{eq:norm-bound-gen} gives $\E_{\x\sim\D}[Q^2(\x)]\leq 4$.
This completes the proof of Claim~\ref{clm:polynomial-apx-gen}, which
completes the proof of Claim~\ref{clm:polynomial-apx-gen}.

\end{proof}
By choosing $\eta=\Theta(\eps^2)$, Claim~\ref{clm:polynomial-apx-gen} contradicts our assumption that $P(\x)$ is $O(\eps^3)$-close to the polynomial $P'(\x)$ that minimizes the $\E_{(\x,y)\sim \D}[(y-P'(\x))^2]$.  This completes the proof.
\end{proof}

The next lemma bounds the dimension of the subspace spanned by the high-influence directions
of a polynomial that minimizes the $L_2$ error with the labels $y$.

Before we proceed to the proof of Theorem~\ref{thm:proper-learner-gen}, we need an algorithm that calculates an approximate minimal polynomial for the Proposition~\ref{prop:structural-gen}.

\begin{lemma}[$L_2$-Polynomial Regression]\label{lem:polynomial-regr-gen}
Let $\D$ be a distribution on $\R^d\times\R$ whose $\x$-marginal is the standard normal 
and whose labels are bounded by $1$. Moreover, let $k \in \Z_+$, and $\eps, \delta>0$. 
There is an algorithm that draws $N=(d k)^{O(k)}\log(1/\delta)/\eps^2$
samples, runs in time $\poly(N,d)$, and outputs a polynomial $P(\x)$ of degree at most $k$ 
such that with probability $1-\delta$ it holds that
$\E_{(\x, y)\sim \D}[(y-P(\x))^2]\leq \min_{P' \in {\cal P}_k} \E_{(\x, y) \sim \D}[(y-P'(\x))^2]+\eps$.
\end{lemma}

\noindent The proof of this lemma is nearly identical to the proof of Lemma~\ref{lem:polynomial-regr}.

We need the following simple fact for ReLUs. An essentially identical fact
was shown in \cite{diakonikolas2020algorithms} Equation (2) for the zero threshold case.
We provide the proof here for completeness.

\begin{fact}\label{fct:approximation-facts-gen}
Let $f_1(\x)=\genf(\vec v\cdot\x+ T)$ and $f_2(\x)=\genf(\vec u\cdot\x+T)$, for $T\in \R$ and $\vec v,\vec u$ unit vectors in $\R^d$. Then $\E_{(\x,y)\sim \D}[(f_1(\x)-f_2(\x))^2]=O(\snorm{2}{\vec v -\vec u}^2)$.
\end{fact}
\begin{proof}
The proof relies on the following fact.
\begin{fact}[Correlated Differences, Lemma 6 of \cite{KTZ19}]
  Let $r(\vec x) \in L_2(\R^d, \normal^d)$ be differentiable almost
  everywhere and let
  \[
    D_\rho =
    \normal\lp(\vec 0,
    \begin{pmatrix}
      \vec I & \rho \vec I \\
      \rho \vec I & \vec I
    \end{pmatrix}
    \rp).
  \]
  We call $\rho$-correlated a pair of random variables $(\vec x, \vec y) \sim
  D_{\rho}$.  It holds
  \[
    \frac{1}{2}
    \E_{(\vec x, \vec z) \sim D_\rho}[(r(\vec x) - r(\vec z))^2]
    \leq (1-\rho) \E_{\vec x \sim \D_\x}\lp[ \snorm{2}{\nabla r(\vec x)}^2 \rp]\, .
  \]
\end{fact}
Using this fact for $\rho=\vec v \cdot \vec u$, and using the approximation $(1-\vec v \cdot \vec u)=\snorm{2}{\vec u-\vec v}^2$ the result follows.
\end{proof}
We also need the following fact about the biases of ReLUs.
\begin{fact}\label{fct:approximation-facts-bias-gen}
Let $f_1(\x)=\genf(\vec v\cdot\x-T)$ and $f_2(\x)=\genf(\vec v\cdot\x - T')$ with $T'\geq T$. 
Then $\E_{(\x,y)\sim \D}[(f_1(\x)-f_2(\x))^2]=O((T-T')^2 + 2 T' (T'-T)e^{-T^2/2})$.
\end{fact}
\begin{proof}
Without loss of generality, we can assume that $\vec v=\vec e_1$. 
The result follows by noting that $\E_{(\x,y)\sim \D}[(f_1(\x)-f_2(\x))^2]\leq \int_{\vec x_1\geq T}(f_1(\x)-f_2(\x))^2 \phi(\x)\d \x+ \int_{ T'}^{T}f_2(\x)\phi(\x)\d \x.$
\end{proof}

We can now prove the main theorem of this section.

\begin{proof}[Proof of Theorem~\ref{thm:proper-learner-gen}]
Let $f\in {\cal C}_0^{\genf}$ such that the $\E_{(\x,y)\sim \D}[f(\x)y]$ is maximized.
Using Lemma~\ref{lem:polynomial-regr-gen} on the labels $y$, with $N=d^{O(1/\eps^{4/3})}\poly(1/\eps)\log(1/\delta)$ samples, we get an $k=O(1/\eps^{4/3})$-degree polynomial $P(\x)$ and it holds that
\[
    \E_{\x\sim \D}[(y-P(\x))^2]\leq \min_{P'\in {\cal P}_k} \E_{\x\sim \D}[(y-P'(\x))^2]+\eps^3\;,
\]
with probability $1-\delta/2$. 
Applying Proposition~\ref{prop:structural-gen} to the polynomial $P(\x)$, 
we get that subspace $V$ spanned by the eigenvectors of the matrix 
$\vec M= \E_{\x\sim \D_\x}[\nabla P(\x)\nabla P(\x)^\top]$ with eigenvalues larger than 
$\eta=\Omega(1/\eps^2)$ contains a vector $\vec v\in V$, so that
\begin{equation}\label{eq:relu_min}
  \min_{\vec v\in V, \|\vec v\|_2 \leq 1, t\in \R} \E_{(\x,y)\sim \D}[(\rho(\vec v \cdot \x + t) -y)^2]
  \leq
  \E_{(\x,y)\sim \D}[(f(\x) - y)^2] + \eps \,.
\end{equation}

Moreover, from Lemma~\ref{lem:dimension-bound}, the dimension of $V$ is $O(1/\eps)^{10/3}$. 
Thus, applying Fact~\ref{fct:cover},
we get that there exists a set $\tilde V$ which is an $\eps$-cover of the set $V$ with respect the $\ell_2$-norm of size $(1/\eps)^{O(1/\eps^{10/3})}$.
 We will use the set ${\cal T}=\{\eps/A,  2\eps/A, \ldots,  1\}$, where $A$ is a large  enough constant,
and show that it is a good cover of the parameter $a$ which is used 
as the guess of the norm of the vector $\vec v$. Finally, we need an effective cover for the biases $t$. 
Observe that from Fact~\ref{fct:approximation-facts-bias-gen}, we need step-size $s=\eps^2/\sqrt{\log(1/\eps)}$ 
and the maximum negative value is $-\Theta(\sqrt{\log(1/\eps)})$. 
(If the value was larger, then the zero function would correlate as well.) 
We also need to bound the maximum positive value. 
We claim that the maximum positive value is some universal constant $C'$. 
This is because the error scales with the norm of the function that is returned by the algorithm 
and because we are trying to be competitive against the unbiased ReLU, 
the norm of the ReLU that is returned cannot be more than 
$2(\E_{\x\sim\D_\x}[\genf^2(\vec w\cdot \x)] + 4\E_{(\vec x, y)\sim \D}[y^2])\leq C'$, 
for some large enough constant $C'$. Thus, the set of biases is
${\cal T}'=\{-C\sqrt{\log(1/\eps)}A/\eps,\ldots,0,s,2s,\ldots, C A/\eps'\}$, 
where we multiply with the minimal value of the guess of the norm. 
This is because if $\snorm{2}{\vec v}=\alpha$, then we have that 
$\genf(\vec v \cdot \x + t)=\alpha\genf(\vec v \cdot \x/\snorm{2}{\vec v} + t/\alpha)$.

We show that the set $\cal H$ is an effective cover, where ${\cal H}=\widetilde{V}\times{\cal T}\times {\cal T}'$.
We show that there exist a set of parameters $(\tilde{\vec v},\tilde a,\tilde t)\in \cal H$ which define a ReLU that correlates with the labels as well as the function $f$. Fix the parameters $\vec v,\alpha,t$ which minimize the Equation~\eqref{eq:relu_min}. Indeed, we have
\begin{align}
\E_{(\x,y)\sim \D} & [a\genf(\vec v \cdot \x + t)-  \tilde a\genf(\tilde{\vec v}\cdot  \x +\tilde {t}))^2]^{1/2} \leq a \E_{(\x,y)\sim \D}[(\genf(\vec v \cdot \x + t)- \genf(\tilde{\vec v}\cdot \x+t))^2]^{1/2} \nonumber\\& +a \E_{(\x,y)\sim \D}[(\genf(\tilde{\vec v} \cdot \x + \tilde t)- \genf(\tilde{\vec v}\cdot \x+t))^2]^{1/2}  +\E_{(\x,y)\sim \D}  [\genf(\tilde{\vec v}\cdot  \x +\tilde {t})^2]^{1/2}|(a-\tilde a)|\;.\label{eq:bound-the-difference-gen}
\end{align}
Applying Facts~\ref{fct:approximation-facts-gen} and \ref{fct:approximation-facts-bias-gen}, we get that
$$
\E_{(\x,y)\sim \D}  [a\genf(\vec v \cdot \x + t)-  \tilde a\genf(\tilde{\vec v}\cdot  \x +\tilde {t}))^2] \leq O(\eps)\;.
$$
Thus, from the triangle inequality, we get that
\begin{equation*}
 \E_{(\x,y)\sim \D}[(\tilde a\rho(\tilde{\vec v} \cdot \x + \tilde t) -y)^2] \leq \E_{(\x,y)\sim \D}[(\rho(\vec v \cdot \x + t) -y)^2] + \E_{(\x,y)\sim \D}  [a\genf(\vec v \cdot \x + t)-  \tilde a\genf(\tilde{\vec v}\cdot  \x +\tilde {t}))^2] 
  \leq
  \E_{(\x,y)\sim \D}[(f(\x) - y)^2] + O(\eps) \,.
\end{equation*}
To complete the proof, it remains to show that Step~\ref{alg:emprical-outputs-gen} 
outputs a hypothesis close to the minimizer inside $\cal H$. 
We need the following claim:
\begin{claim}\label{clm:hoefdi} Let $h\in{\cal C}^\genf$ and let $\widehat \D$ be the empirical distribution with $N=O((1/\eps^2)\log(1/\delta))$ samples. Then, with probability $1-\delta$, it holds
\[
|\E_{(\x,y)\sim \widehat\D}[(y-h(\x))^2] -\E_{(\x,y)\sim\D}[(y-h(\x))^2] |\leq \eps\;.
\]
\end{claim}
\begin{proof}
We need first to prove that with probability $1-\delta$ it holds:
$$|\E_{(\x,y)\sim \widehat\D}[yh(\x)] -\E_{(\x,y)\sim\D}[yh(\x)] |\leq \eps\;.$$
Using Markov's inequality, we have
\begin{align*}
\pr[|\E_{(\x,y)\sim \widehat\D}[h(\x)y]- \E_{(\x,y)\sim \D}[h(\x)y]|\geq \eps]
& \leq \frac{1}{N\eps^2}\var[h(\x)y] \\         
& \leq \frac{1}{N\eps^2}\E_{(\x,y)\sim \D}[h^2(\x)y^2]\\
& \leq O\left(\frac{1}{N\eps^2}\right)\;,
\end{align*}
where we used the fact that our functions are bounded in $L_2$-norm. 
With the same procedure we bound the difference 
$|\E_{(\x,y)\sim \widehat\D}[h^2(\x)] -\E_{(\x,y)\sim\D}[h^2(\x)] |\leq \eps$. 
By using the fact that $N=O(1/\eps^2)$, we get our result for constant probability. 
By applying a standard probability amplification technique, 
we can boost the confidence to $1-\delta$ with $N'=O(N\log(1/\delta))$ samples.
\end{proof}

Therefore, from Claim~\ref{clm:hoefdi}, it follows that $O(\frac{1}{\eps^2}\log({\cal H}/\delta))$ 
samples are sufficient to guarantee that the excess error of the chosen hypothesis 
is at most $\eps$ with probability at least $1-\delta/2$.

To bound the runtime of the algorithm, we note that $L_2$-regression has runtime
$d^{O(1/\eps^{4/3})}\poly(1/\eps)\log(1/\delta)$ and the exhaustive search 
over an $\eps$-cover takes time $(1/\eps)^{O(1/\eps^{10/3})} \log(1/\delta)$ time.  
The total runtime of our algorithm in the case
where $1/\eps^{10/3} \leq d$ is
\[
  \Big( d^{O(1/\eps^{4/3})} + (1/\eps)^{O(1/\eps^{10/3})} \Big) \log(1/\delta) \,.
\]
In the case where $1/\eps^{10/3}> d$, 
one can directly do a brute-force search over an $\eps$-cover 
of the $d$-dimensional unit ball: we do not need to perform
our dimension-reduction process and the runtime is bounded above by the previous case.
\end{proof}

\appendix

\end{document}